\documentclass{article} % For LaTeX2e
\usepackage{iclr2019_conference,times}

\usepackage[utf8]{inputenc} % allow utf-8 input
\usepackage{url}            % simple URL typesetting
\usepackage{booktabs}       % professional-quality tables
\usepackage{amsfonts}       % blackboard math symbols
\usepackage{nicefrac}       % compact symbols for 1/2, etc.
\usepackage{microtype}      % microtypography
\usepackage{booktabs}
\usepackage{amsmath,amssymb}
\usepackage{mathtools}
\usepackage{bm}
\usepackage{xcolor}
\usepackage{amsthm}
\usepackage{wrapfig}
\newtheorem{prop}{Proposition}[]
\usepackage{comment} 

\usepackage{color}
\usepackage{graphicx} % more modern
\usepackage{subfigure}

\newlength{\sfigwidth}
\newcommand{\LL}{\mathcal{L}}
\newcommand{\TT}{\mathcal{T}}
\newcommand{\todo}[1]{\textcolor{red}{TODO: #1}} 
\newcommand{\hidemelater}[1]{{#1}}

\renewcommand{\todo}[1]{} % todo's will be removed from camera-ready
\renewcommand{\hidemelater}[1]{} % hidemelater's will be removed from camera-ready
\usepackage{afterpage}
\usepackage{algorithm} % remove for ICML
\usepackage{algorithmicx} % remove for ICML
\usepackage{algcompatible}
\usepackage[noend]{algpseudocode} % removes EndIf, EndWhile etc
\algrenewcommand\algorithmicindent{0.5em} % reset indent, default is 1.5em I think

\algnewcommand{\LineComment}[1]{\(\triangleright\) \textit{#1}}

\usepackage{amsmath,amsfonts,bm}

\def\eqref#1{equation~\ref{#1}}

\def\1{\bm{1}}

\def\vtheta{{\bm{\theta}}}
\def\vphi{{\bm{\phi}}}

\DeclareMathAlphabet{\mathsfit}{\encodingdefault}{\sfdefault}{m}{sl}
\SetMathAlphabet{\mathsfit}{bold}{\encodingdefault}{\sfdefault}{bx}{n}

\DeclareMathOperator*{\argmin}{arg\,min}

\DeclareMathOperator{\sign}{sign}

\usepackage{pdftexcmds}  % needed for fnpct to load properly https://tex.stackexchange.com/questions/384448/fnpct-package-in-tex-live-2017-not-working/384452
\usepackage[dont-mess-around]{fnpct}
\usepackage{etoolbox}
\makeatletter
\patchcmd\maketitle{\setcounter{footnote}{0}}{}{}{}
\patchcmd\maketitle{%
 \renewcommand\thefootnote{\@fnsymbol\c@footnote}}{\AdaptNote\thanks\multthanks}{}{}
\patchcmd\maketitle{%
 \def\@makefnmark{\rlap{\@textsuperscript{\normalfont\@thefnmark}}}}{}{}{}
\makeatother
\definecolor{mydarkblue}{rgb}{0,0.08,0.45}
\usepackage[colorlinks,citecolor=mydarkblue,urlcolor=mydarkblue,linkcolor=mydarkblue,filecolor=mydarkblue]{hyperref}

\usepackage[compact]{titlesec}
\titlespacing{\section}{0pt}{2ex}{1ex}
\titlespacing{\subsection}{0pt}{1ex}{0ex}
\titlespacing{\subsubsection}{0pt}{0.5ex}{0ex}
\newcommand{\reducevspace}{\vspace*{-0.5em}} % nips
\newcommand{\mainlabel}{{main}}
\newcommand{\auxlabel}{{aux}}
\newcommand{\maintask}{\TT_{\mainlabel}}
\newcommand{\auxtask}{\TT_{\auxlabel}}
\newcommand{\mainloss}{\LL_{\mainlabel}}
\newcommand{\auxloss}{\LL_{\auxlabel}}
\newcommand{\iterid}{t}
\newcommand{\iterrltime}{{t'}}
\newcommand{\iterrl}{i} % used for RL other than t
\newcommand{\iter}{{(\iterid)}}
\newcommand{\iterminusone}{{(\iterid-1)}}
\newcommand{\iterplusone}{{(\iterid+1)}}

\newcommand{\paramshared}{\vtheta}

\newcommand{\paramsall}{\mathbf{\Theta}}
\newcommand{\paramsone}{\vphi_{\mainlabel}}
\newcommand{\paramstwo}{\vphi_{\auxlabel}}

\title{
Adapting Auxiliary Losses Using Gradient Similarity
}

\author{Yunshu Du\thanks{Equal Contribution} ~\thanks{Washington State University. Work done during an internship at DeepMind.} , \
Wojciech M. Czarnecki\footnotemark[1] ~\thanks{DeepMind} , \  
Siddhant M. Jayakumar\footnotemark[3] , \
Mehrdad Farajtabar\footnotemark[3] , \AND
Razvan Pascanu\footnotemark[3] , \
Balaji Lakshminarayanan\thanks{Google Research, Brain Team. Work done while at DeepMind.}
\\ \\
{
\small \texttt{yunshu.du@wsu.edu}, \texttt{\{lejlot, sidmj, farajtabar, razp, balajiln\}@google.com}
}
}

\iclrfinalcopy % Uncomment for camera-ready version, but NOT for submission.

\begin{document}

\maketitle % for arxiv, NIPS versions

\begin{abstract}
One approach to deal with the statistical inefficiency of neural networks is to rely on additional auxiliary losses that help build useful representations. However, it is not always trivial to know if an auxiliary task will be helpful or, more importantly, when it could start hurting the main task. We investigate whether the cosine similarity between the gradients of the different tasks provides a signal to detect when an auxiliary loss is helpful to the main loss. We prove theoretically that our approach is at least guaranteed to converge to the critical points of the main task. Then, we demonstrate the practical usefulness of the proposed algorithm in both supervised learning and reinforcement learning domains. Finally, we discuss the potential implications of relying on this heuristic as a measurement of task similarity.

\end{abstract}

\section{Introduction}

Neural networks are powerful function approximators that have excelled on a wide range of tasks \citep{simonyan2014very, mnih2015human, he2016deep, alphago, vaswani2017attention}. Despite the state of the art results across domains, they remain data-inefficient and expensive to train. In supervised learning, large deep learning benchmarks with millions of examples are needed for training \citep{russakovsky2015imagenet} and the additional implication of requiring human intervention to label a large dataset can be prohibitively expensive. In reinforcement learning (RL), agents typically consume millions of frames of experiences before learning to act reasonably in complex environments \citep{alphago, impala}, which not only puts pressure on computing power but also makes particular domains (e.g., robotics) impractical.  

Different techniques have been studied for improving data efficiency, from data augmentation \citep{krizhevsky2012imagenet, simonyan2014very, hauberg2016dreaming} to multi-task learning \citep{distral, kendall2017multi, gradnorm, popart, sener2018multi} to transfer learning \citep{taylor2009transfer, pan2010survey}. In this work, we focus on a particular type of transfer learning: transferring knowledge of an \emph{auxiliary task} to a \emph{main task}. We assume that besides the main task, one has access to one or more additional auxiliary tasks that share some unknown structures with the main task. To improve data efficiency, these additional tasks can be used as auxiliary losses to transfer knowledge to the main task. Note that \textit{only the performance of the main task is of interest}, even though the model is trained simultaneously on all tasks; any improvement on the auxiliary losses is useful only to the extent that it helps learning features or behaviors for the main task. %

Auxiliary tasks have been shown to work well for reinforcement learning \citep[e.g.,][]{zhang2016augmenting,jaderberg2017unreal, mirowski2017learning,  papoudakis18, li2019cross}. 
The success of these approaches depend on how well aligned the auxiliary losses are with the main task. Knowing this apriori is typically non-trivial and the usefulness of an auxiliary task can change through the course of training. In this work, we explore a simple yet effective heuristic for measuring the similarity between an auxiliary task and the main task of interest (given the value of their parameters) using gradient information.

\section{Notation and Problem Description}
\label{sec:notation}

Assume we have a main task $\maintask$ and some update rule induced by an auxiliary task $\auxtask$. This can be given by the gradients of an auxiliary loss $\auxloss$ (though our derivation holds also when there is no such loss, as will be discussed later). Let us assume the main task induces a loss $\mainloss$. We care only about maximizing performance on $\maintask$ while $\auxtask$ is an auxiliary task which is not of direct interest. %
The goal is to devise an algorithm that/ can automatically: 
\begin{enumerate}
    \item leverage $\auxtask$ as long as it is helpful to $\maintask$, %
    \item detect when $\auxtask$ becomes harmful to $\maintask$ and block negative transfer to recover the performance of training only on $\maintask$. %
\end{enumerate} 

We propose to parameterize the solution for $\maintask$ and $\auxtask$ by two neural networks, $f(\cdot, \paramshared, \paramsone)$ and $g(\cdot, \paramshared, \paramstwo)$, where they share a subset of parameters (denoted by $\paramshared$) and have their own set of parameters (denoted by $\paramsone$ and $\paramstwo$ respectively). Generally, the auxiliary loss literature proposes to minimize weighted losses of the form 
\begin{align}
    \argmin_{\paramshared,\paramsone,\paramstwo} \mainloss(\paramshared, \paramsone) + \lambda \auxloss(\paramshared, \paramstwo)
    \label{eq:old-objective}
\end{align}
under the intuition that modifying $\paramshared$ to minimize $\auxloss$ will improve $\mainloss$ if the two tasks are sufficiently related. There are two potential problems with minimizing \eqref{eq:old-objective}: (i) $\lambda$ is usually kept as a constant which is not adaptable where the usefulness of $\auxtask$ changes during training (e.g., helps initially but hurts later), and (ii) the \emph{relatedness} of the two tasks are often ambiguous to quantify.

Ideally one would want to modulate the weight $\lambda$ at each learning iteration $\iterid$ such that either final performance on main task is maximized or convergence speed of the main task is maximized.\footnote{Note that both quantities might be ill-defined without considering a fixed budget of iterations. Also, while they are correlated, the two maximization problems can lead to different behaviors.}

A greedy variant of this objective is to modulate  $\lambda$ at each learning iteration $\iterid$ by how useful $\auxtask$ is for $\maintask$ given $\paramshared^\iter, \paramsone^\iter, \paramstwo^\iter$. That is, at each optimization iteration, we want to efficiently approximate the solution to 
\begin{align}
        \argmin_{\lambda^\iter}\mainloss\Bigl( \paramshared^\iter -\alpha \nabla_\paramshared (\mainloss + \lambda^\iter \auxloss),  %\nonumber \\
        \paramsone^\iter - \alpha \nabla_{\paramsone} \mainloss\Bigr).
    \label{eq:objective}
\end{align}
Note that the input space of $\maintask$ and $\auxtask$ do not have to match. In particular, $\auxtask$ does not need to be defined for an input of $\maintask$ or the other way around.\footnote{In the supervised learning case when the input features are shared, this resembles \citet{quadrianto2010multi}, \textit{multi-task learning without label correspondences} setting.} Solving the general problem, or even the greedy variant given by \eqref{eq:objective} is expensive, as it requires at each step to solve a non-linear maximization problem. %
Instead, we look for a cheap heuristic to approximate $\lambda^\iter$ which does not require hyperparameter tuning or computing derivatives through the learning process---and which will outperform a \emph{constant} $\lambda^\iter$.%

Note that our setup is in sharp contrast to multi-objective optimization where both the tasks are of interest. We are only interested in minimizing the $\mainloss$ with the help of the $\auxloss$ where possible. 

\section{Gradient Cosine Similarity}
\label{sec:cos-between-task}

\begin{figure}[ht]
\centering
\includegraphics[width=0.6\textwidth]{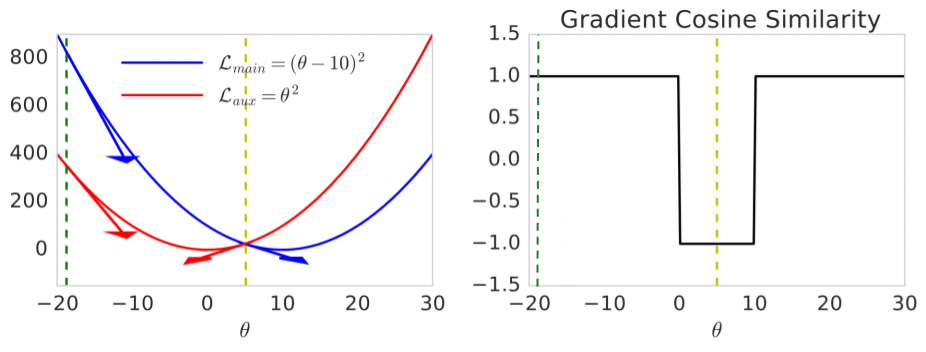}
\caption{Illustration of cosine similarity between gradients on synthetic loss surfaces. }
\label{fig:grad-similarity-toy}
\end{figure}

We propose to use the \emph{cosine similarity of gradients between tasks} to quantify the relatedness between tasks and hence for approximating $\lambda^\iter$. Consider the example in Figure~\ref{fig:grad-similarity-toy} where the main function to minimize is $\mainloss=(\theta-10)^2$ and the auxiliary function is $\auxloss=\theta^2$, their gradients are $\nabla_\theta\mainloss=2(\theta-10)$ and $\nabla_\theta\auxloss=2\theta$ respectively. When %
$\theta=-20$, the gradients of the main and auxiliary functions point in the same direction and the cosine similarity is $1$; minimizing the auxiliary loss is beneficial for minimizing the main loss. However, at a different point, $\theta=5$, the two gradients point in different directions and the cosine similarity is $-1$; minimizing the auxiliary loss would hinder minimizing the main loss. 
This example suggests a natural strategy for approximating $\lambda^\iter$: \emph{minimize the auxiliary loss as long as its gradient has non-negative cosine similarity with the main gradient; otherwise, the auxiliary loss should be ignored}. %
This follows the well-known intuition that if a vector is in the same half-space as the gradient of a function $f$, then it is a \emph{descent direction} for $f$. This reduces our strategy to ask \emph{if the gradient of the auxiliary loss is also a descent direction for the main loss of interest}.

We start by first introducing two propositions in this section and prove that our proposed method ensures convergence on the main task---although one can not guarantee the speed of convergence %
(we discuss this in more details later in this section). The proof holds even when adding to the gradient vectors that might not be the gradient of any function, which we argue is a practical scenario in reinforcement learning. Later in Section \ref{sec:results}, we show empirically that our method is a good heuristic for blocking negative transfer and in practice often leads to positive transfer.

\begin{prop} \label{prop:converge} Given any gradient vector field $G(\paramshared) = \nabla_\paramshared \LL(\paramshared)$ and any vector field $V(\paramshared)$ (such as the gradient of another loss function, or an arbitrary set of updates), an update rule of the form
 \begin{align}
 \paramshared^\iterplusone := \paramshared^\iter - &\alpha^\iter \Bigl( G(\paramshared^\iter) 
 + V(\paramshared^\iter) \max\bigl(0, \cos(G(\paramshared^\iter), V(\paramshared^\iter) )\bigr)\Bigr) 
\end{align}

converges to the local minimum of $\LL$ given small enough $\alpha^\iter$.
\end{prop}

Proof is provided in Appendix~\ref{sec:proof1}. %

We point out three important properties of the above statement. First, %
it guarantees only lack of divergence, but does not guarantee any improvement of convergence. That is, cosine similarity is not a silver bullet that guarantees positive transfer, but it can drop the ``worst-case scenarios'' thus preventing negative transfer. In principle, the convergence speed of the main loss could be affected both positively or negatively, %
{though we notice empirically that the effect tends to be positive in realistic scenarios.} %
Second, it is worth noting that %
simply adding an arbitrary vector field to $\nabla \mathcal{L}$ does not have the convergence property and hence Proposition \ref{prop:converge} has practical utility---there are realistic scenarios where one does not always form a gradient field. For example, the update rule of the Q-learning algorithm in RL does not form a gradient field. Another example is when adding  another  gradient field  to the update rule of REINFORCE (which is a Monte-Carlo estimate of a gradient), depending on the setting, the sum might not be a gradient field anymore.
Lastly, it should be noted that gradient similarity is a local quantity on which we can not technically rely to draw global conclusions on learning dynamics. That is, having negative cosine similarity at step $t$ does not guarantee divergence of the main loss, nor does it guarantee that using the auxiliary loss (with fixed weight) will lead to slower convergence than not having it at all. %We will expand on this in the discussion of the paper. For a toy example highlighting some of these failures please see Appendix~\ref{sec:failuremodes}.

%\begin{comment}
\paragraph{A toy example.}
%\razp{To save space we can consider moving the toy example to appendix}
We illustrate both positive and negative scenarios of our proposed method by running a set of experiments using steepest descent method. There are two main losses to be minimized: $L_1(\theta_1,\theta_2)=\theta_1^2+\theta_2^2$ and $L_2(\theta_1,\theta_2)=(\theta_1<0)(\theta_1^2+\theta_2^2)+(\theta_1>0)(1-\exp(-2(\theta_1^2+\theta_2^2)))$. Two auxiliary losses are: $L_3(\theta_1,\theta_2) = (\theta_1-1)^2 + (\theta_2-1)^2$ and $L_4(\theta_1, \theta_2)=(\theta_1-2)^2 + (\theta_2-0.5)^2$. We also consider an arbitrary vector field $V (\theta_1, \theta_2) = [ -\tfrac{\theta_2}{\theta_1^2+\theta_2^2}-2\theta_1,\tfrac{\theta_1}{\theta_1^2+\theta_2^2} -2\theta_2 ]$. 
Each experiment runs for $600$ iterations with a constant step size of $0.01$. Convergence time is defined as the number of steps needed to get below value $0.1$ of the main loss. Each colored curve shows an example trajectory. 

\begin{wrapfigure}{l}{0.61\textwidth}
\centering
    \includegraphics[width=.6\textwidth]{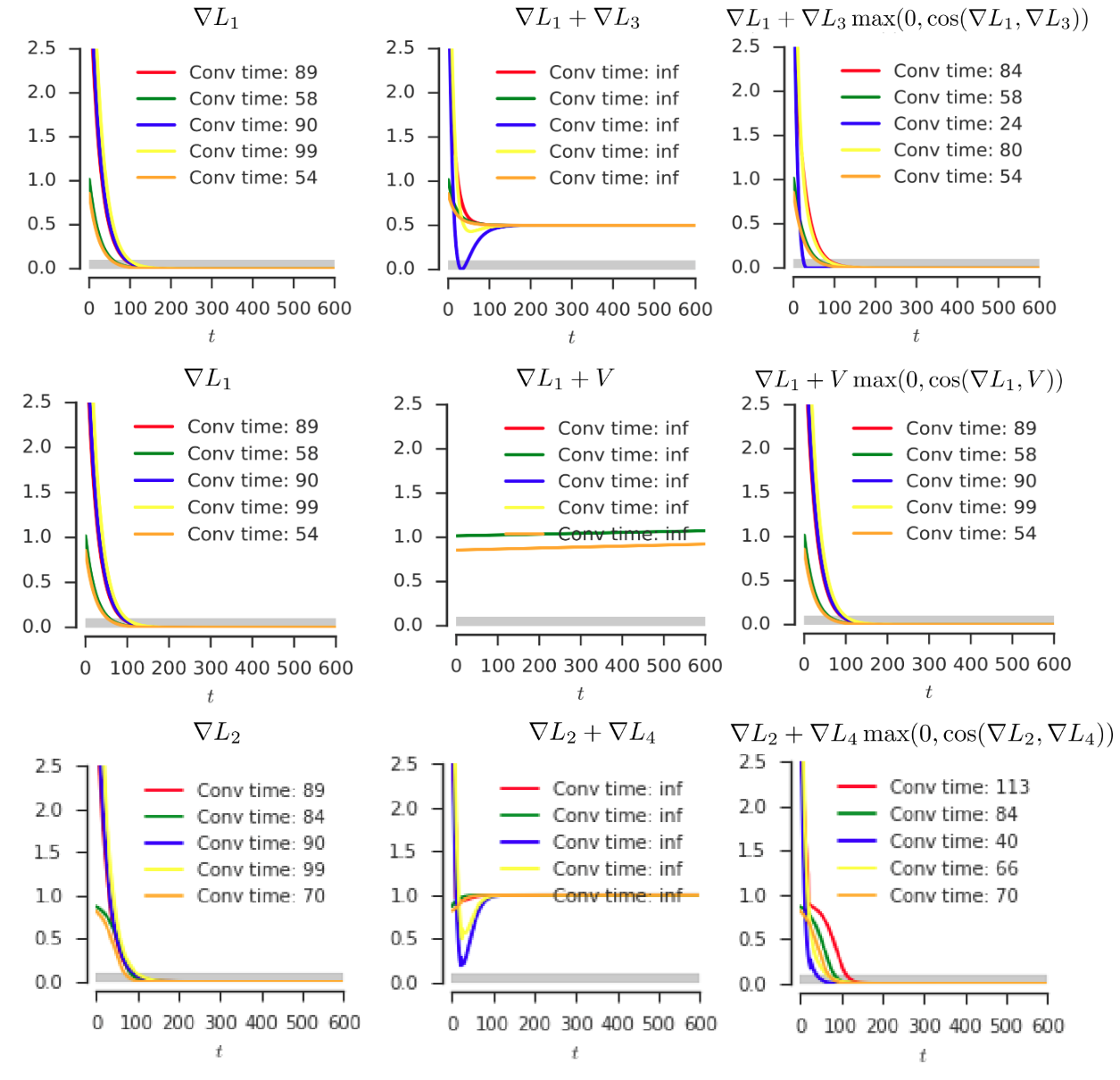}
\caption{
    %Consider two example loss functions as the main losses to be minimized:
    Positive and negative examples of our proposed method. %$L_1(\theta_1,\theta_2)=\theta_1^2+\theta_2^2$ and $L_2(\theta_1,\theta_2)=(\theta_1<0)(\theta_1^2+\theta_2^2)+(\theta_1>0)(1-\exp(-2(\theta_1^2+\theta_2^2)))$. 
    \emph{Top row:} combine $L_1$ with the gradient of an auxiliary loss $L_3$. 
    %$L_3(\theta_1,\theta_2) = (\theta_1-1)^2 + (\theta_2-1)^2$. 
    \emph{Middle row:} combine $L_1$ with a vector field $V$. %$V (\theta_1, \theta_2) = [ -\tfrac{\theta_2}{\theta_1^2+\theta_2^2}-2\theta_1,\tfrac{\theta_1}{\theta_1^2+\theta_2^2} -2\theta_2 ]$. 
    \emph{Bottom row:} combine $L_2$ with the gradient of an auxiliary loss $L_4$. %$L_4(\theta_1, \theta_2)=(\theta_1-2)^2 + (\theta_2-0.5)^2$.
    Our method (the last column) converges in all cases, while simply adding a gradient or vector field leads to divergence (the second column). %The top two rows shows when the auxiliary loss is helpful to the main loss while the bottom row shows a counter example where the auxiliary loss slows down minimizing the main loss.
    }
    \label{fig:resuts_quadratics}
\end{wrapfigure} 

Shown in Figure~\ref{fig:resuts_quadratics}, the first column depicts the loss when following the update given by $\nabla L_1$ and $\nabla L_2$ respectively. The second column shows the performance on the main loss when using an update that sums between the gradient of the main loss and the vector fields produced by $\nabla L_3, V,$ and $ \nabla L_4$ respectively. The last column shows the performance on the main loss when using our proposed cosine similarity method. Figure~\ref{fig:resuts_quadratics} shows our method ensures convergence in all cases. In particular, the top two rows show the positive case where our method speeds up learning, even when combined with $V$. Observe that simply adding $\nabla L_1 + V$ is not a gradient vector field and following it leads to divergence of $L_1$. The last row shows a negative case where our method slows down convergence. Note, however, that our method still converges, while simply adding $\nabla L_2$ and $\nabla L_4$ (the second column) diverges. More details are provided in Appendix~\ref{sec:failuremodes} to help intuitively understand the kind of scenarios for which our approach could help.

Proposition \ref{prop:converge} refers to losses with the same set of parameters $\paramshared$, while \eqref{eq:objective} refers to the scenario when each loss has task specific parameters (e.g. $\paramsone$ and $\paramstwo$). The following proposition extends to this scenario:

\begin{prop} 
Given two losses parametrized with $\paramsall$ (some of which are shared $\paramshared$ and some unique to each loss $\paramsone$ and $\paramstwo$), learning rule:
\begin{align}
    \paramshared^\iterplusone :=  \paramshared^\iter - \alpha^\iter \Bigl( %
    &\nabla_{\paramshared} \mainloss(\paramshared)+
    \\ &  
    \nabla_{\paramshared} \auxloss(\paramshared) 
    \max(0, 
    \cos(
    \nabla_{\paramshared} \mainloss(\paramshared), %
    \nabla_{\paramshared} \auxloss(\paramshared) ))\Bigr) \nonumber
\end{align}
\reducevspace
\begin{align}
    \paramsone^{\iterplusone}  &:= \paramsone^{\iter} - \alpha^{\iter} \nabla_{\paramsone} \mainloss(\paramsall) \\ %
    \paramstwo^\iterplusone &:= \paramstwo^\iter - \alpha^\iter \nabla_{\paramstwo} \auxloss(\paramsall) 
\end{align}
leads to convergence to local minimum of $\mainloss$ w.r.t. $(\paramshared, \paramsone)$ given small enough $\alpha^\iter$.
\end{prop}

\begin{proof}
Comes directly from the previous proposition that $G = \nabla_{\paramshared} \mainloss$ and $V = \nabla_{\paramshared} \auxloss$. For any vector fields $A,B,C$, we have $\langle A, B \rangle \geq 0$ and $\langle C, B \rangle \geq 0$ implies $\langle A+C, B \rangle \geq 0$.
\end{proof}

Analogous guarantees also hold for the \emph{unweighted} version of this algorithm. Instead of weighting by $\cos(G,V)$, we use a binary weight $(\sign(\cos(G,V)) + 1)/2$ which is equivalent to using $V$ iff $\cos(G,V) > 0$. When training with mini-batches, accurately estimating $\cos(G,V)$ can be difficult due to noise; the unweighted variant only requires $\sign(\cos(G,V))$ which can be estimated more robustly. Hence, we use the unweighted variant in our experiments unless otherwise specified. %

Despite its simplicity, our proposed update rule can give rise to  interesting phenomena. We can show that the emerging vector field could be non-conservative, which means there does not exist a loss function for which it is a gradient. While this might seem problematic (for gradient-descent-based optimizers), %
it describes only the global structure---typically used optimizers are local in nature and they do local, linear or quadratic approximations of the function~\citep{shwartz2017opening}. Consequently, in practice, one should not expect any negative effects from this phenomena, as it simply shows that our proposed technique is in fact qualitatively changing the nature of the update rules for training.

\begin{prop}
In general, the proposed update rule does not have to create a conservative vector field.
\end{prop}

Proof is provided in Appendix~\ref{sec:proof3}.

\section{Applications}
\label{sec:results}

We now demonstrate how to use the gradient cosine similarity to decide when to leverage the auxiliary. We aim to answer two questions: (i) is gradient cosine similarity capable of detecting negative interference, which can be changing during the course of training between tasks, and (ii) compared with a constant weight, can our proposed heuristic block negative transfer when $\auxtask$ starts to hinder $\maintask$. %
All experiments (unless otherwise stated) follow the \emph{unweighted} version of our method, summarized in Algorithm~\ref{alg:usecosine}. The \emph{weighted} version of our method is summarized in Algorithm~\ref{alg:usecosine2}, Appendix~\ref{sec:pseudocode}. 

\begin{algorithm}[ht]
 \caption{Unweighted version of our method.} 
 \label{alg:usecosine}
 \begin{algorithmic}[1]
 \State Initialize shared parameters $\paramshared$ and task specific parameters $\paramsone, \paramstwo$ 
 randomly.
 \For{$\mathsf{iter}=1:\mathsf{max\_iter}$}
     \State Compute $\nabla_{\paramshared} \mainloss$, $\nabla_{\paramsone} \mainloss$, $\nabla_{\paramshared} \auxloss$, $\nabla_{\paramstwo} \auxloss$.
     \State Update $\paramsone$ and $\paramstwo$ using corresponding gradients
     \If{$\cos(\nabla_{\paramshared} \mainloss, \nabla_{\paramshared} \auxloss) \geq$  0} 
     \State Update $\paramshared$ using $\nabla_{\paramshared} \mainloss + \nabla_{\paramshared} \auxloss$
     \Else
          \State Update $\paramshared$ using $\nabla_{\paramshared} \mainloss$
     \EndIf
 \EndFor
 \end{algorithmic}
\end{algorithm}

\subsection{Binary Classification Tasks}

First, we apply our method in supervised learning and design a multi-task binary classification problem on ImageNet \citep{russakovsky2015imagenet}. We
take a pair of classes from ImageNet, refer to these as class $A$ and class $B$; all the other 998 classes (except $A$ and $B$) are referred to as the $background$. Our tasks $\maintask$ and $\auxtask$ are then formed as a binary classification of \emph{if an image is class $A$ (otherwise $background$)} and \emph{if an image is class $B$ (otherwise $background$)} respectively. 

Our goal is to show that cosine similarity can automatically detect when the auxiliary task becomes unhelpful and block negative transfer thus we consider two scenarios: (i) auxiliary task helps and (ii) auxiliary task hurts. It is natural to hypothesis that a pair of \emph{near} classes would fit scenario (i) and a pair of \emph{far} classes fit scenario (ii). However, there are no ground truth labels for such a class distance measure thus it is difficult to decided if two classes are similar or not. To combat this, we instead use two distance measures to serve as an \emph{estimated} class similarity for selecting class pair $A$ and $B$: \emph{lowest common ancestor (LCA)} is the ImageNet label hierarchy, and \emph{Frechet Inception Distance (FID)} \citep{fid} is the image embedding from a pre-trained model.  
Based on these measures, we picked three pair of classes for \emph{near}, class $871$ (trimaran) vs. $484$ (catamaran), $250$ (Siberian husky) vs. $249$ (malamute), and $238$ (Greater Swiss Mountain dog) vs. $241$ (Entleucher); and for \emph{far}, class $920$ (traffic light) vs. $62$ (rock python), $926$ (hotpot) vs. $800$ (slot), and $48$ (Komodo dragon) vs. $920$ (traffic light). See Appendix~\ref{sec:imagenet:classes:picking} for details on class pair selection.  %

We train modified ResNetV2-18 model \citep{he2016identity} with all parameters in the convolutional layers shared (denote as $\paramshared$) between tasks, followed by task-specific parameters $\paramsone$ and $\paramstwo$.  
First, we use a multi-task learning setup and minimize $\mainloss+\auxloss$ (i.e., $\lambda=1$), and measure cosine similarity on $\paramshared$ through the course of training. Figure~\ref{fig:imagenet-pairs-3} shows that cosine similarity is higher for \emph{near} pairs (blue lines) and lower (mostly negative) for \emph{far} pairs (red lines), indicating that cosine similarity captures task relatedness well. Next, we compare single-task training (i.e., $\lambda=0$), multi-task training, and our proposed method %
on \emph{near} and \emph{far} class pairs to illustrate scenario (i) and (ii) respectfully. %
Note that, the multi-task baseline here is not minimizing the desired objective, $\mainloss$ and is only used for demonstrating what would happen if the auxiliary loss is weighted as a constant without any adaptions based on its usefulness.

\begin{figure*}[ht]
\begin{center}
    \begin{subfigure}[
    Cosine similarity: near (blue) \& far pairs (red). %
    ]{
    \includegraphics[width=0.36\textwidth]{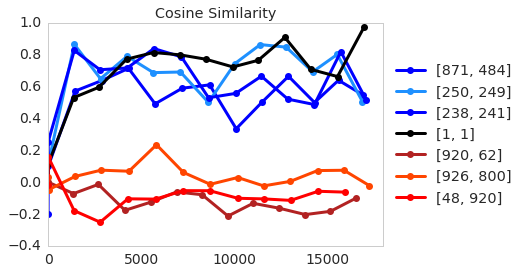}
    \label{fig:imagenet-pairs-3}
    }\end{subfigure}
    \begin{subfigure}[Near pair: $871$ vs $484$]{
    \includegraphics[width=0.29\textwidth]{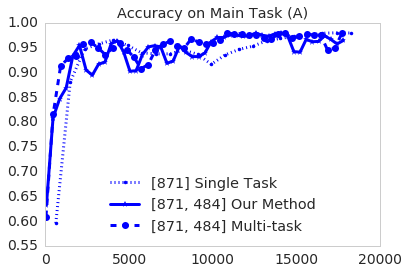}
    \label{fig:nearpair}
    }\end{subfigure}
    \begin{subfigure}[Far pair: $48$ vs $920$]{
    \includegraphics[width=0.29\textwidth]{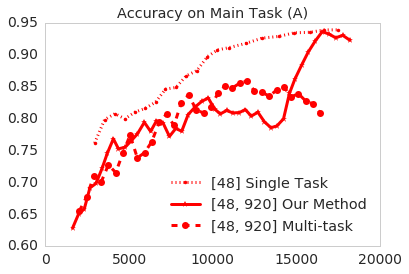}
    \label{fig:farpair}
    }\end{subfigure}
\caption{Experiments on ImageNet class pairs. \emph{(a)}: gradient cosine similarity is higher for near pairs and lower for far pairs. \emph{(b)} and \emph{(c)}: testing accuracy on single-task (dotted), multi-task (dashed), and our method (solid). %
}
 \label{fig:imagenet:subsets:transfer}
\end{center}
\end{figure*}

Figure~\ref{fig:nearpair} shows scenario (i) where auxiliary task helps. While all variants perform similarly in terms of final performance, our method obtains an initial boost compared to the single task learning and thereafter achieves similar performance as the multi-task learning. This behavior indicates that when gradient cosine similarity is high, our method automatically ``switches on'' $\auxloss$ and trains as multi-tasking. Figure~\ref{fig:farpair} shows scenario (ii) where the auxiliary task could hinder. %
Here, the multi-task baseline suffers from poorer performance than single-task learning throughout the training process %
due to the class distance is far and lack of transferability. %
In contrast, our method blocks negative transfer by ``switching off'' the unhelpful auxiliary task and later recovers the performance of single-task. %
We note that the blocking did not happen until later in training (at around step $15,000$). This observation potentially reflects two aspects of our method: (i) estimating the cosine similarity between gradients can be noisy and unreliable, and (ii) our method does not guarantee maximizing transfer but only ``drop the worst.'' Nevertheless, this experiment shows that our method can eventually notice that the auxiliary task hinders the main task and is able to block negative transfer then recover the performance of the single task learning.

\subsection{Multi Class Classification Tasks}

Second, we evaluate our method on the multi-class classification task of rotated MNIST digits, a common benchmark in multi-task continual learning~\citep{lopez2017gradient,farajtabar2019orthogonal} and multi-task generalization~\citep{ghifary2015domain} problems. In this setting, the main task is classifying the original MNIST images to one of the 10 digit classes. The auxiliary task consists of all images rotated by a certain degree. For this task we use a simple multi-layer fully connected network
with 3 hidden layers each with 100 neurons and ReLU  nonlinearity on the top. Each task is then is equipped with a separate 10 neuron head for the associated 10 dimensional logit.  We used all the 60K training inputs and train the network for 50 epochs with batch size set to 128. For the optimizer we used RMSprop with learning rate $0.001$. %

\begin{table}[ht]
\centering
\caption{Results on RotatedMNIST: Main task is the original MNIST problem. In the auxiliary task all the input images are rotated by the  degree specified in the rotation column.}
\footnotesize
\begin{tabular}{c|ccc}
\toprule
            & \multicolumn{3}{c}{ Error $\pm$ Std ($\%$)}        \\ \midrule
rotation    & single-task     & multi-task      & our method      \\ \midrule
0           & $2.14 \pm	0.16$ &	$2.19 \pm 0.19$ & $2.11 \pm 0.13$ \\
45          & $2.15 \pm	0.15$ &	$2.36 \pm 0.17$ & $2.10 \pm 0.14$ \\
90          & $2.17 \pm	0.19$ &	$2.41 \pm 0.22$ & $2.26 \pm 0.08$ \\
135         & $2.11 \pm	0.12$ &	$2.50 \pm 0.19$ & $2.37 \pm 0.12$ \\
180         & $2.20 \pm 0.16$ & $2.23 \pm 0.12$ & $2.21 \pm 0.16$ \\
\bottomrule
\end{tabular}
\label{tab:rotated-mnist}
\end{table}

Table~\ref{tab:rotated-mnist} shows the results. We reported the mean and standard deviation of the classification error on 10k held-out test examples over 10 runs with random network initialization and data order. Overall, the proposed method shows a strong performance compared to single- and multi-task baselines. More specifically, in the case of 0 rotation (i.e. the two tasks are similar) the auxiliary task has positive influence in both cases over the single-task with  some marginal advantage of the proposed method compared to the multi-task baseline. Interestingly, when the digits are rotated (45, 90, etc) the auxiliary task is negatively interfering with the main task. Again, the proposed method succeeds to mitigate the negative behavior significantly. For the case of 45 degrees it not only outperform the multi-task baseline but also beats the single task again by benefiting from positive transfer while cleverly avoiding negative interference. In the case of 90 and 135 degrees the negative effect of the auxiliary task compromised its positive effect and the performance degrades, however, as expected the proposed method is avoiding a significant amount of negative interference compared to the multi-task baseline.

\subsection{Gridworld Tasks}

We now turn to the reinforcement learning domain and consider a typical RL problem where one aims to find a policy $\pi$ that maximizes the sum of future discounted rewards $\mathbb{E}_{\pi}  [ \sum_{\iterrltime=1}^N \gamma^{\iterrltime-1} r_\iterrltime  ]$ in a partially observable Markov Decision Process (POMDP). There have been many techniques proposed to solve this optimization problem (e.g., \citep{reinforce}, \citep{Qlearning}, \citep{PPO}, \citep{impala}). %
Inherently, these techniques are data inefficient due to the complexity of the problem. One way to address this issue is to use transfer learning, such as transfer from pre-trained policies~\citep{policydistillation}. %
However, a teacher policy is not always available for the main task. When in this scenario, one can train policies in other tasks that share enough similarities and hope for a positive transfer. One way of exploiting this extra information is to use %
distillation~\citep{knowledgedistillation, policydistillation} to guide the main task in its initial learning phase~\citep{kickstarting}---although it might be difficult to find a suitable strategy that combines the main and auxiliary losses and/or smoothly transition between them. Typically, the teacher policy can be treated as an auxiliary loss~\citep{kickstarting} or a prior~\citep{distral} with a fixed mixing coefficient. However, these techniques become unsound if the teacher policy is helpful only in specific states while hindering in other states.

We design a simple RL experiment to show that our method is capable of finding the strategy of combining the main loss and the auxiliary loss. We define a distribution over a set of $15\times15$ gridworlds, where an agent observes its surrounding (up to four pixels away) and can move in four directions, up, down, left, and right. %
We randomly place %
two types of positive rewards, $+5$ and $+10$ points, both terminating an episode. %
To guarantee a finite length of episodes, we add a fixed probability of $0.01$ of transitioning to a non-rewarding terminal state. See Appendix~\ref{sec:maze} for more experiment details and a visualization of the environment. 

First, we train a Q-learning agent to navigate the gridworld which gives us a \emph{teacher} policy $\pi^\text{Q}$. Then, we create tasks to which there is a possible positive knowledge transfer by keeping the environment with the same layout but removing the $+10$ rewards (and corresponding states are no longer terminating). Consequently, we have two tasks: the auxiliary task $\auxtask$ where we have a strong \emph{teacher} policy $\pi^\text{Q}$, and the main task $\maintask$ where the $+10$ rewards are removed. %
{That is, we know $\auxtask$ is helpful initially but hurts later since $\maintask$ no longer has the $+10$ reward; our proposed method should adapt to this change during training.} We sample $1,000$ such environment pairs and report expected returns obtained ($100$ evaluation episodes per evaluation point) using five training methods: \textbf{1) reward} using only policy gradient \citep{reinforce} in the new task, this is the baseline;
\textbf{2) distill} using only distillation \citep{policydistillation} cost towards the teacher;
\textbf{3) add} simply adding the two above;
\textbf{4) cos-weighted} using the weighted version of our method (Algorithm~\ref{alg:usecosine2});
\textbf{5) cos-unweighted} using the unweighted version of our method (Algorithm~\ref{alg:usecosine}).

\begin{figure*}[th]
\begin{center}
\centering
\textbf{Cross-task transfer experiment $\auxtask \rightarrow \maintask$}\par\medskip
\includegraphics[width=\textwidth]{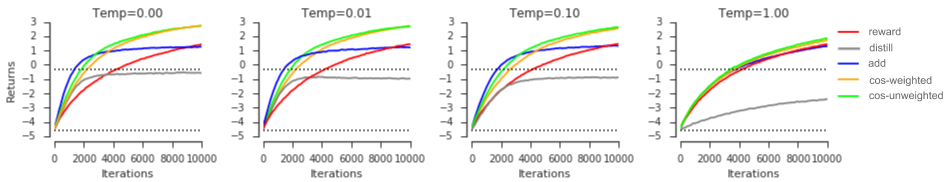} \\
\textbf{Distilling from the solution of the same task $\maintask \rightarrow \maintask$}\par\medskip
\includegraphics[width=\textwidth]{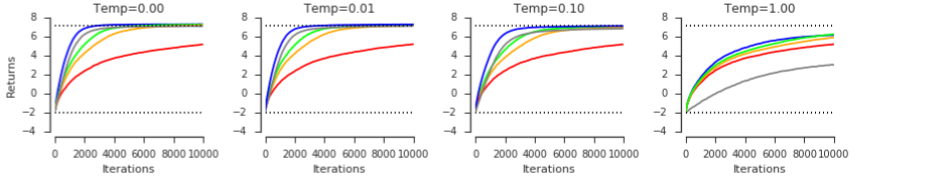}

\caption{ \emph{Top row:} expected learning curves for cross-environment distillation experiments, averaged over $1,000$ partially observable gridworlds. The teacher's policy is based on Q-Learning, its performance in a new environment (with modified positive rewards) is represented by the top dotted line. The bottom dotted line represents random policy. Each column represents a different temperature applied to the teacher policy. $0$ temperature is the original deterministic greedy policy given by Q-Learning. \emph{Bottom row:} expected learning curves for same-environment distillation experiments when the teacher is perfect, where, trusting the teacher everywhere is optimum.
}
\label{fig:distillation}
\end{center}
\end{figure*}

Results are shown in the top row of Figure~\ref{fig:distillation}. %
The baseline \textbf{reward} gives a score of slightly above $1$ point after $10,000$ steps of training. %
To leverage teacher policies, we define the auxiliary loss to be a distillation loss (i.e., a per-state cross-entropy between the teacher's and student's distributions over actions). First, we test using solely the distillation loss while sampling trajectories from the student (\textbf{distill}). This recovers a subset of teacher's behaviors and end up with $0$ point---an expected negative transfer as the teacher is guiding us to states that are no longer rewarding. Then, we test simply adding gradients estimated by policy gradient and distillation (\textbf{add}). The resulting policy learns quickly but saturates at a return of $1$ point, showing very limited positive transfer. Lastly, when using our proposed gradient cosine similarity as the measure of transferability  (\textbf{cos-weighted} and \textbf{cos-unweighted}),
we get a significant performance boost. The learned policies reach baseline performance after just one-third of steps taken by the baseline, and on average obtain $3$ points after $10,000$ steps.\footnote{Note that we compute cosine similarity between a distillation gradient and a single sample of the policy gradient estimator, meaning that we are using a high variance estimate of the similarity. For larger experiments, one would need to compute running means for reliable statistics.} %

This experiment shows that gradient cosine similarity  %
allows using knowledge from other related tasks in an automatic fashion. The agent is simply ignoring the teacher signal when it disagrees with its policy gradient estimator. If they do agree in terms of which actions to reinforce, the teacher's logits are used for better replication of useful policies. In addition, in the bottom row of Figure~\ref{fig:distillation}, we present an experiment of transferring between the same task $\maintask$. Here, the teacher is perfect and it is the optimal to follow it everywhere. We see that the cosine similarity methods underperformed that of simply adding the two losses. This is expected as the noise in the gradients makes it hard to measure if the two tasks are a good fit. %

\subsection{Single and Multi-task Atari Games}

Finally, we consider a similar RL setup in the more complex Atari domain \citep{bellemare2013arcade}. We use a deep RL agent the same as previous works %
\citep{mnih2015human, mnih2016asynchronous, impala, popart} and train using the batched actor-critic with V-trace algorithm \citep{impala}. See Appendix~\ref{sec:atari} for more experiment details.

First, we look at training an agent on a main task (here, to play Breakout) given a \emph{sub-optimal} teacher (obtained by stop training before it converges) to the task. Analogous to the previous experiment, we define the auxiliary loss as the distillation (i.e., the KL divergence) between the sub-optimal teacher and the training model. Figure \ref{fig:atari-0}) shows as expected that solely relying on distillation loss (\textbf{Only KL}) leads to lower performance. Training with both distillation and RL losses without adaptation (\textbf{RL+KL(Baseline)}) leads to slightly better but also sub-optimal performance. While both approaches learn very quickly initially, they plateau much lower than the pure RL approach (\textbf{RL(Baseline)}, minimizing just the main task) due to the potential negative effect from the imperfect teacher and the lack of ability to detect and prevent such an effect. In our method (\textbf{RL+KL(Our Method)}), the KL penalty is scaled at every time-step by the cosine similarity between the policy gradient and distillation losses; once this falls below a fixed \emph{threshold}, the loss is ``turned off'' thus preventing negative transfer. This experiment shows our approach is able to learn quickly at the beginning then continue fine-tuning with pure RL loss once the distillation loss is zeroed out. That is, %
when the teacher does not provide useful information anymore, our method simply encourages the agent to learn on its own. %

% \begin{figure}[ht]
\begin{wrapfigure}{r}{0.4\textwidth}
\begin{center}
\includegraphics[width=0.35\textwidth]{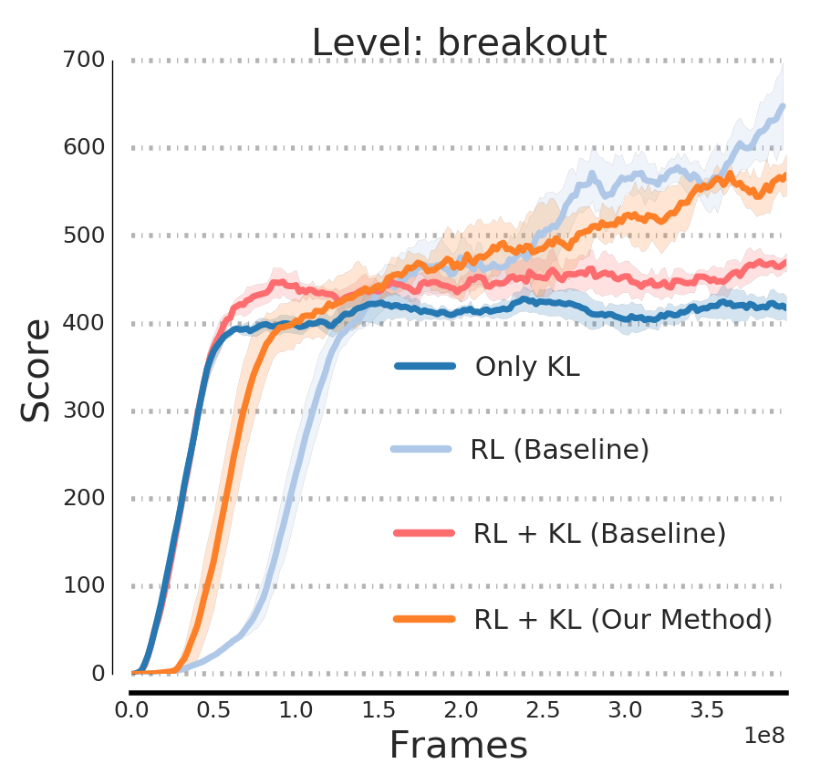}
\caption{Results on Breakout. We perform distillation of a sub-optimal teacher policy as an auxiliary task.}
\label{fig:atari-0}
\end{center}
\end{wrapfigure}
% \end{figure}

Then, we consider a setting where the main task $\maintask$ is to train an agent to play two Atari games, Breakout and Ms.~PacMan, such that the performance on both games are jointly maximized (i.e., $\mainloss = \LL_{Breakout} + \LL_{PacMan}$). %
Similar to previous experiments, we have access to a teacher trained on just Breakout as the auxiliary task. Note that $\maintask$ itself is chosen to be
multi-task to illustrate a complex scenario where $\auxtask$ helps with only part of $\maintask$ (here, Breakout).

Figure \ref{fig:atari-1} shows that, compared to the baseline method which trains only on the main loss (\textbf{Multitask}) and the simple addition of the updates of the two tasks method (\textbf{Multitask RL + Distillation}) where the agent learns one task at the expense of the other, our method (\textbf{Multitask RL + Distillation (Our Method)}) of scaling the auxiliary loss by gradient cosine similarity is able to correct itself by ``turning off'' the auxiliary distillation when the teacher is no longer helpful %
and is able to learn Ms. PacMan without forgetting Breakout. The evolution of the gradient cosine similarity between the auxiliary teacher and the main Breakout and Ms.PacMan task in Figure \ref{fig:atari-1} provides a meaningful cue for the usefulness of $\auxloss$.

\begin{figure*}[ht]
\begin{center}
\includegraphics[width=\textwidth]{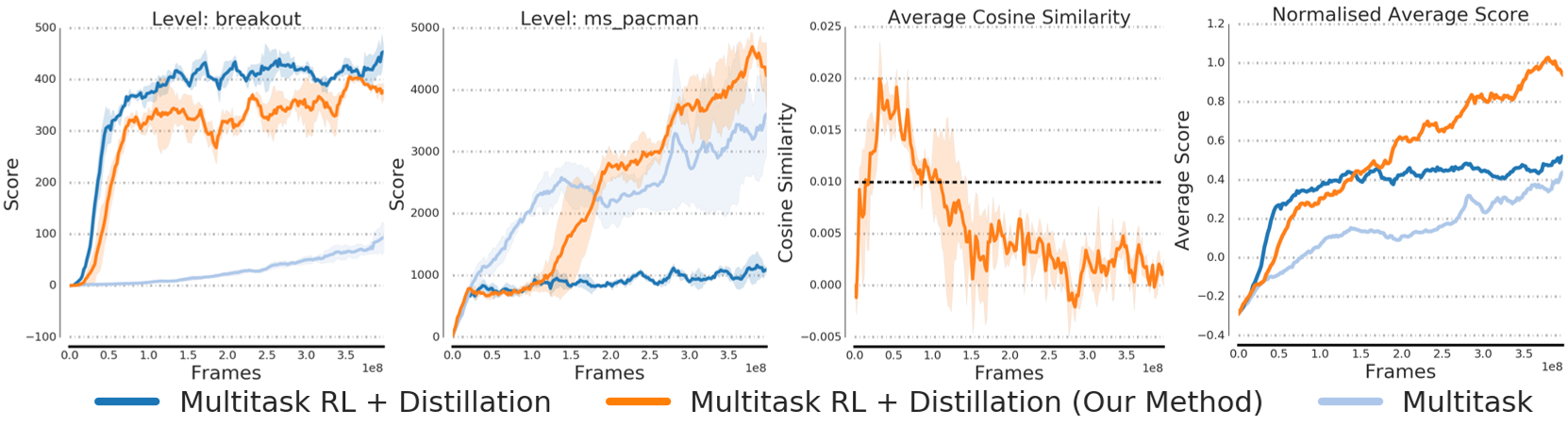}
\end{center}
\caption{Results on Breakout and Ms. PacMan (averaged over $3$ seeds). The two plots to the left show performance on Breakout and Ms. PacMan respectively. The third plot shows how the gradient cosine similarity between the teacher and the two tasks changes during training. The last plot shows an average score of the multi-task performance (normalized independently for each game based on the best score achieved across all experiments). Our method learns both games without forgetting and achieves the best average performance.}
\label{fig:atari-1}
\end{figure*} 

\section{Related Work}

Our work is related to the literature on identifying task similarity in transfer learning. It is generally believed that positive transfer can be achieved when source task(s) and target task(s) are related. However, it is usually assumed that this relatedness mapping is provided by human experts \citep{taylor2009transfer, pan2010survey}; few works have addressed the problem of finding a general measure of similarity to predict transferability between tasks. In RL, methods have been proposed to use the Markov Decision Process (MDP) similarity as a measure of task relatedness \citep{carroll2005task, ammar2014automated}. In image classification, \citet{yosinski2014transferable} defined image similarity in ImageNet by manually splitting classes into man-made versus natural objects. However, none of these works have explicitly used the learned similarity metric to quantify the transferability from one task to another. In our work, we propose to use cosine similarity of gradients as a generalizable measure across domains and show it can be directly leveraged to detect when unwanted interference occurs and block negative transfer. One important aspect of task similarity for transfer is that it is highly dependent on the parametrization of the model and the current value of the parameters. We exploit this property by providing a heuristic similarity measure for the current parameters, resulting in an approach that relies on an adaptive weight over the updates of the model.

Auxiliary tasks have shown to be beneficial in facilitating learning across domains. In RL, several work have studied using different tasks as auxiliaries. The UNREAL framework \citep{jaderberg2017unreal} incorporated unsupervised control along with reward prediction learning as auxiliary tasks; \citet{mirowski2017learning} studied auxiliary tasks in the context of navigation; \citet{kartal2019terminal} used terminal prediction as an auxiliary task; \citet{hernandez2019agent} considered multiagent system where the other agents' policy can be used as auxiliary tasks for the main agent. In image classification, \citet{zhang2016augmenting} used unsupervised reconstruction tasks. %
\citet{papoudakis18} also explored auxiliary losses for VizDoom. However, these works rely on empirical results and do not address how the auxiliary tasks were selected. %
In our work, we propose a simple yet effective heuristic that provides explicit guidance on the important question of how to select a good auxiliary task. %

Our work is also related to, but different from multi-task learning \citep{caruana1997multitask}, particularly the line of work on using adaptive scaling techniques for multi-objective learning. %
\emph{GradNorm} \citep{gradnorm} uses gradient magnitude to scale loss function for each task, aiming to learn well for all tasks. \citet{kendall2017multi} proposed a weighting mechanism by considering the homoscedastic uncertainty of each task. \citet{sener2018multi} propose formulating multi-task learning as multi-objective optimization and leverage multiple gradient descent algorithm (MGDA) \citep{mgda} to find a common descent direction among multiple tasks. Our work distinct in two ways: (i) we care only about the performance of the main task but not all tasks; our optimization goal is %
not the traditional multi-objective optimization, and (ii) %
the aforementioned work scale the losses individually without looking at their interaction (which can lead to poor performance in our problem setup when the auxiliary task hurts the main task), 
whereas we look for alignments in the vector field between the main and the auxiliary task, and the auxiliary task is used only when it is well-aligned with the main task. 
Nevertheless, our approaches are complementary since they are solving different problems and it would be interesting to combine them to further improve the performance.

\section{Shortcomings}
\label{sec:discussion}

We discuss here a few shortcomings of our proposed method and present some initial analyses. First, estimating the cosine similarity between the gradients of tasks could be expensive or noisy and we currently use a fixed threshold for turning off the auxiliary loss. These could be addressed by calculating a running average of the cosine similarity to get a smoother result and potentially hyper-tune the threshold instead of setting it as a fixed constant. %
{We point out also that even for a perfect alignment---a cosine similarity of $1$---is not informative. In this case, using the auxiliary loss would be equivalent to increasing the learning rate. The positive transfer comes mostly from the fact that the gradients are \emph{not} perfectly aligned, meaning that the auxiliary loss is learning features that could become useful to the main loss later on. This is reflective on the fact that the gradient descent is not necessarily the optimal descent direction.} 
In addition, one might argue that our approach would fail in high-dimensional spaces since random vectors in such spaces tend to be orthogonal, so that cosine similarity will be naturally driven to $0$. In fact, this is not the case; if two gradients are meant to be co-linear, the noise components cancel each other thus will not affect the cosine similarity estimation. We empirically explore this in Appendix~\ref{sec:cosine-high-dimensions} in the supplementary material.

Second, the new loss surface might be less smooth which can be problematic when using optimizers that rely on statistics of the gradients or second order information (e.g. Adam or RMSprop). In these cases, the transition from just the gradient of the main task to the sum of the gradients can affect the statistics of the optimizer in unexpected ways. While this can be technically true, we have not observed this behaviour in practice. 

Thirdly, although the proposed approach works well empirically on complex and noisy tasks like Atari games, as discussed in Section~\ref{sec:cos-between-task}, it guarantees only  the main task's convergence, but not how fast it is. While removing the worst case scenarios is important and a good first step, one might care more for data efficiency when using auxiliary losses (i.e., faster convergence). In Appendix~\ref{sec:failuremodes} we provide counter examples where the proposed update rule slows down learning, compared to optimizing the main task alone. %

Finally, similar to other concurrent works \citep{yu2020gradient,schaul2019ray} we are making the assumption that the auxiliary loss harming learning of main loss locally (at time $t$) implies it will affect the convergence speed globally. In theory, however, this is not necessarily the case. By moving in an ascent direction on the main loss, one might learn better representations that could lead to better performance or faster learning. Empirically we do not notice this to happen and we hypothesize that overparametrization helps making gradient similarity reliable as a signal for task interference. We provide a longer discussion in the Appendix~\ref{sec:overparam}.

\section{Conclusions}

In this work, we explored a simple yet efficient technique to ensure that an auxiliary loss does not hurt the learning on the main task. The proposed approach reduces to applying gradients of the auxiliary task only if they are a descent direction of the main task. 
We have empirically shown the potential of using the proposed hypothesis as an elegant way of picking a suitable auxiliary task. %
While we have mostly considered scenarios where the auxiliary task helps initially but hurts later, it would be interesting to explore settings where the auxiliary task hurts initially but helps in the end. Examples of such are annealing $\beta$ in $\beta$-VAE \citep{betavae} and annealing the confidence penalty in \citet{pereyra2017regularizing}.

\clearpage
\newpage

\bibliography{main}
\bibliographystyle{abbrvnat}

\clearpage
\newpage

\onecolumn

\appendix
{\begin{center} {\Large{\textbf{Adapting Auxiliary Losses Using Gradient Similarity: \\ Supplementary Material}}}
\end{center}}

\section{Proofs}

\subsection{Proof for Proposition 1}
\label{sec:proof1}

\noindent \textit{ Given any gradient vector field $G(\paramshared) = \nabla_\paramshared \LL(\paramshared)$ and any vector field $V(\paramshared)$ (such as gradient of another loss function, but could be arbitrary set of updates), an update rule of the form}
\begin{align}
 \paramshared^\iterplusone := \paramshared^\iter - \alpha^\iter ( G(\paramshared^\iter) + V(\paramshared^\iter) \max(0, \cos(G(\paramshared^\iter), V(\paramshared^\iter) ))
\end{align}
\textit{converges to the local minimum of $\LL$ given small enough $\alpha^\iter$.}

\begin{proof}
Let us denote
$$
G^\iter := G(\paramshared^\iter)\;\;\;\;\;\;\;\; V^\iter := V(\paramshared^\iter)\;\;\;\;\;\;\;\; \nabla \LL^\iter := \nabla_\paramshared \LL(\paramshared^\iter)
$$
$$
\Delta\paramshared^\iter :=  G^\iter + V^\iter \max(0, \cos(G^\iter, V^\iter )).
$$
Our update rule is simply
$\paramshared^\iterplusone := \paramshared^\iter - \alpha^\iter \Delta\paramshared^\iter$
and we have
\begin{align}
\langle \Delta\paramshared^\iter, \nabla \LL^\iter \rangle &= 
\langle G^\iter + V^\iter \max(0, \cos(G^\iter, V^\iter )),  \nabla \LL^\iter  \rangle \\
&= \langle G^\iter, \nabla \LL^\iter \rangle + \langle V^\iter \max(0, \cos(G^\iter, V^\iter) ),  \nabla \LL^\iter  \rangle \\
&= \| \nabla \LL^\iter \|^2 + \tfrac{1}{\| V^\iter \| \| \nabla \LL^\iter \|} \max(0, \langle \nabla \LL^\iter , V^\iter \rangle) \langle V^\iter ,  \nabla \LL^\iter  \rangle \geq 0.
\end{align}
And it can be 0 if and only if $\| \nabla \mathcal{L}^\iter \| = 0$ (since sum of two non-negative terms is zero iff both are zero, and step from (4) to (5) is only possible if this is not true), thus it is 0 only when we are at the critical point of $\mathcal{L}$.
Thus the method converges due to convergence of steepest descent methods, see 
``Cauchy's method of minimization'' \citep{goldstein1962cauchy}. 
\end{proof}

\subsection{Proof for Proposition 3}
\label{sec:proof3}

\noindent \textit{
In general, the proposed update rule  does not have to create a conservative vector field.
}
\begin{proof}
Proof comes from a counterexample, let us define in 2D space:
$$
\mathcal{L}_{\mainlabel}(\theta_1, \theta_2) = a \theta_1
$$
$$
\mathcal{L}_{\auxlabel}(\theta_1, \theta_2) = \left \{  \begin{matrix}
a \theta_1& \;\; \text{if } \theta_1 \in [1, 2] \wedge \theta_2 \in [0, 1] \\
0 &\text{ therwise }
\end{matrix} \right .
$$
for some fixed $a \neq 0$.
Let us now define two paths (parametrized by $s$) between points $(0,0)$ and $(2,2)$, path $A$ which is a concatenation of a line from $(0,0)$ to $(0,2)$ (we call it $U$, since it goes up) and line from $(0,2)$ to $(2,2)$ (which we call $R$ as it goes right), and path $B$ which first goes right and then up. Let $V_\text{cos}$ denote the update rule we follow, then:
$$
\int_A V_\text{cos} ds = \int_A \nabla \mathcal{L}_\text{\mainlabel} ds = \int_U \nabla  \mathcal{L}_\text{\mainlabel} ds + \int_R \nabla  \mathcal{L}_\text{\mainlabel} ds =  \int_R \nabla  \mathcal{L}_\text{\mainlabel} ds
= 2a
$$
At the same time, since gradient of $\mathcal{L}_{\mainlabel}$ is conservative by definition:
$$
\int_B V_\text{cos} ds = \int_B \nabla \mathcal{L}_\text{\mainlabel} ds + \int_C \nabla \mathcal{L}_\text{\auxlabel} ds
=
\int_A \nabla \mathcal{L}_\text{\mainlabel} ds + \int_C \nabla \mathcal{L}_\text{\auxlabel} ds 
=
2a + \int_C \nabla \mathcal{L}_\text{\auxlabel} ds =  3a
$$
where $C$ is a part of $B$ that goes through $[1, 2] \times [0, 1]$. We conclude that $\int_A V_\text{cos} ds \neq \int_B V_\text{cos} ds$, so our vector field is not path invariant, thus by  Green's Theorem it is not conservative, which concludes the proof. See Figure~\ref{fig:pathintegrals} for visualization.
\end{proof}

\begin{figure}
    \centering
    \includegraphics[width=0.7\textwidth]{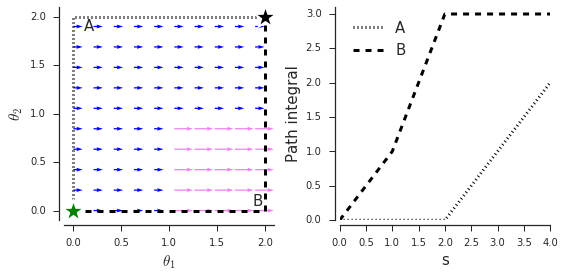}
    \caption{Visualization of the counterexample from Proposition 3, stars denote starting (green) and end (black) points. Dotted and dashed lines correspond to paths A and B respectively. Blue arrows represent gradient vector field of the main loss, while the violet ones the merged vector field.}
    \label{fig:pathintegrals}
\end{figure}

\begin{figure*}%
    \centering
    \includegraphics[width=\sfigwidth]{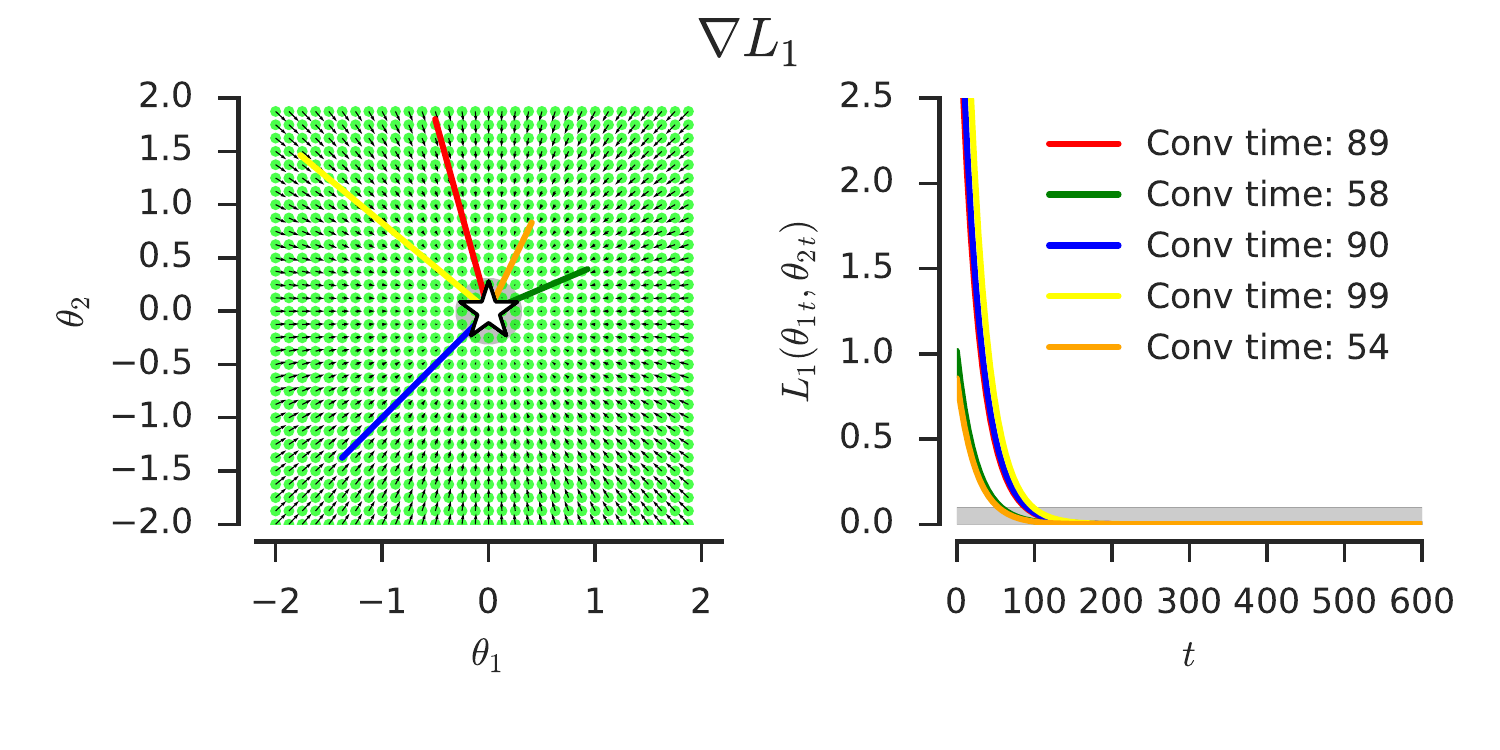}
    \includegraphics[width=\sfigwidth]{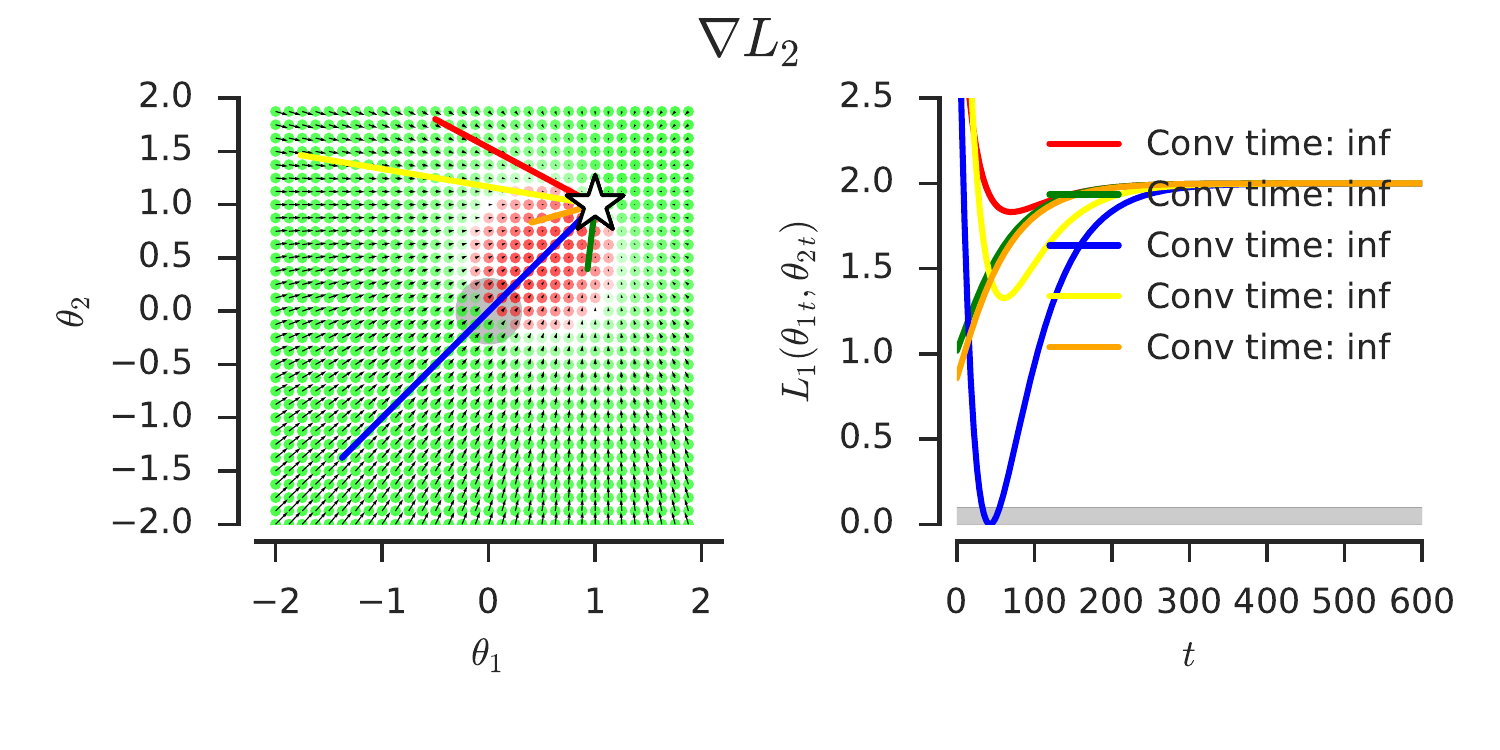}
    \includegraphics[width=\sfigwidth]{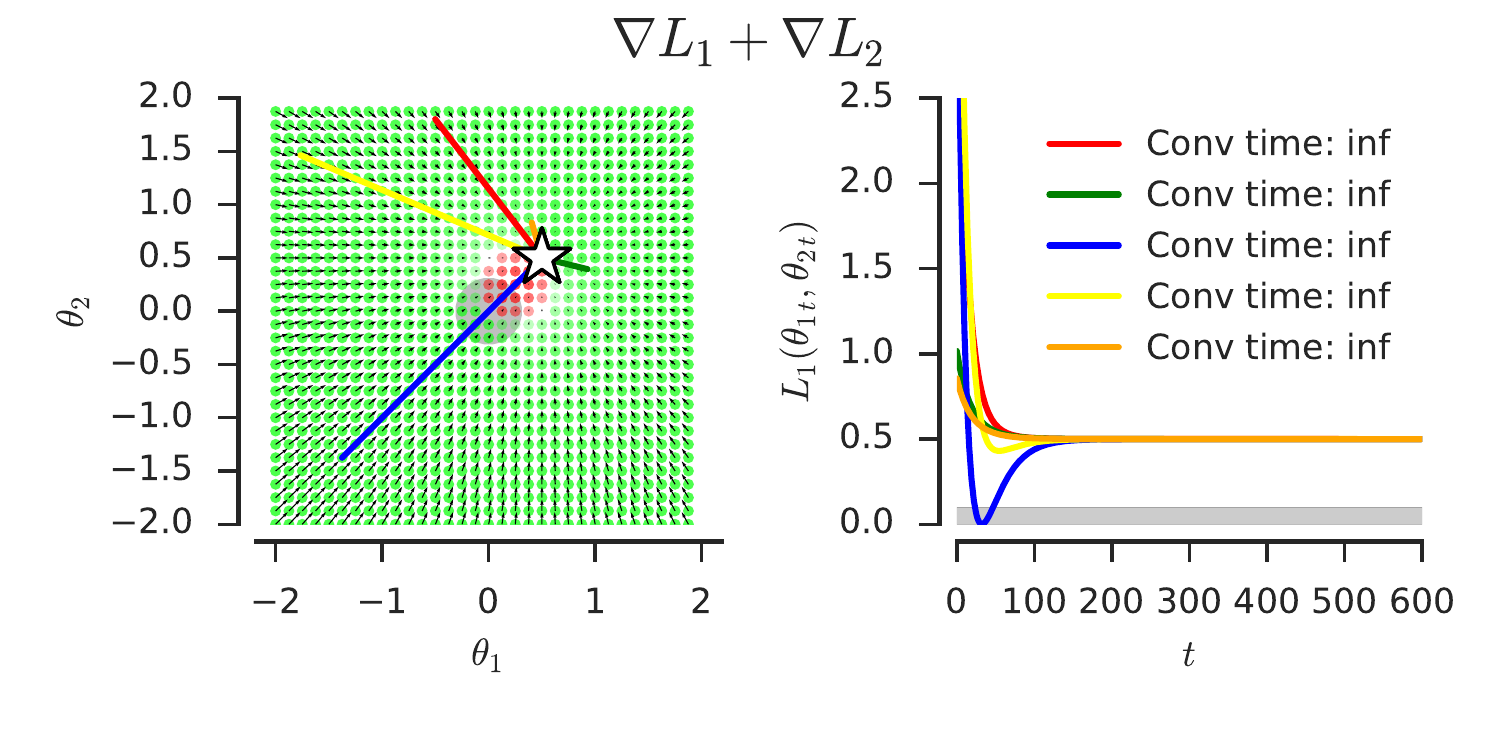}
    \includegraphics[width=\sfigwidth]{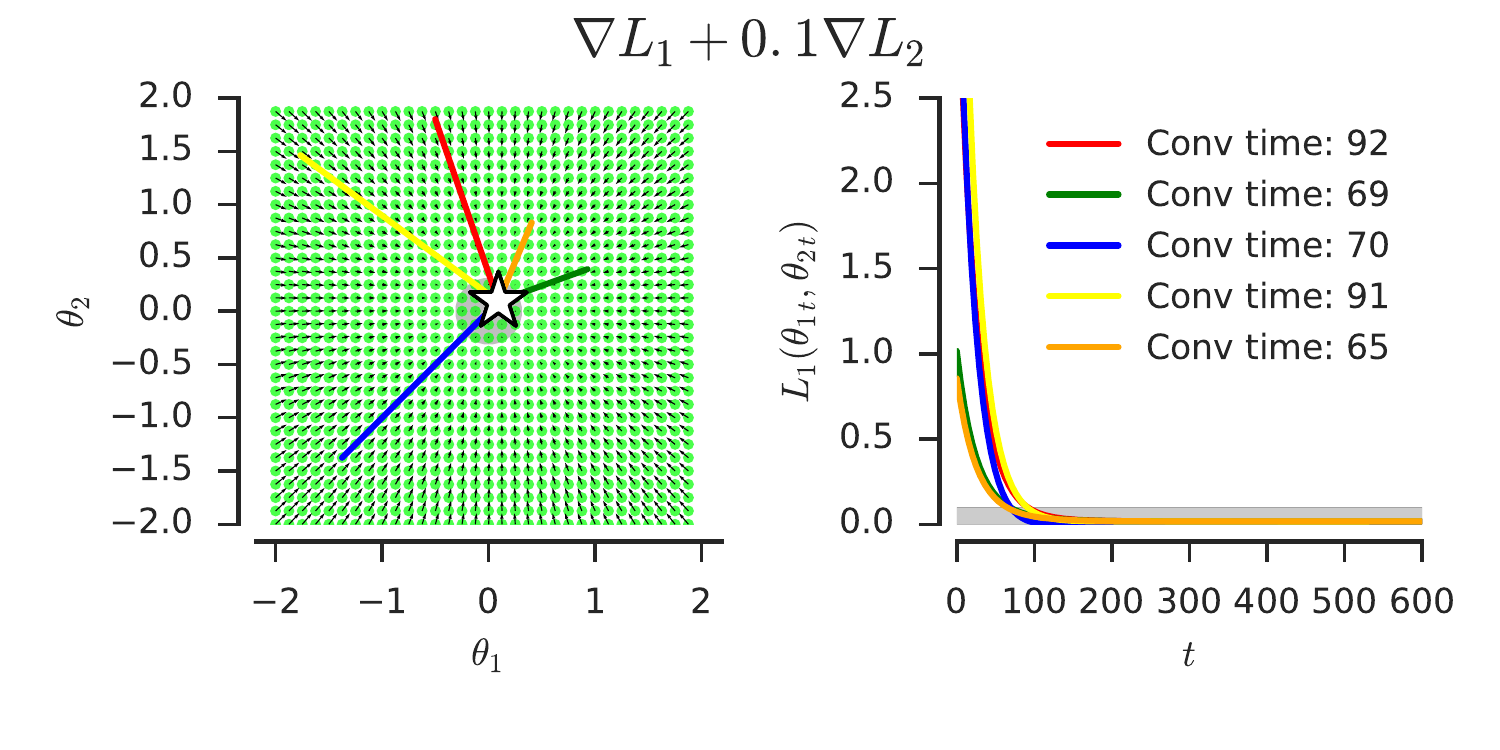}
    \includegraphics[width=\sfigwidth]{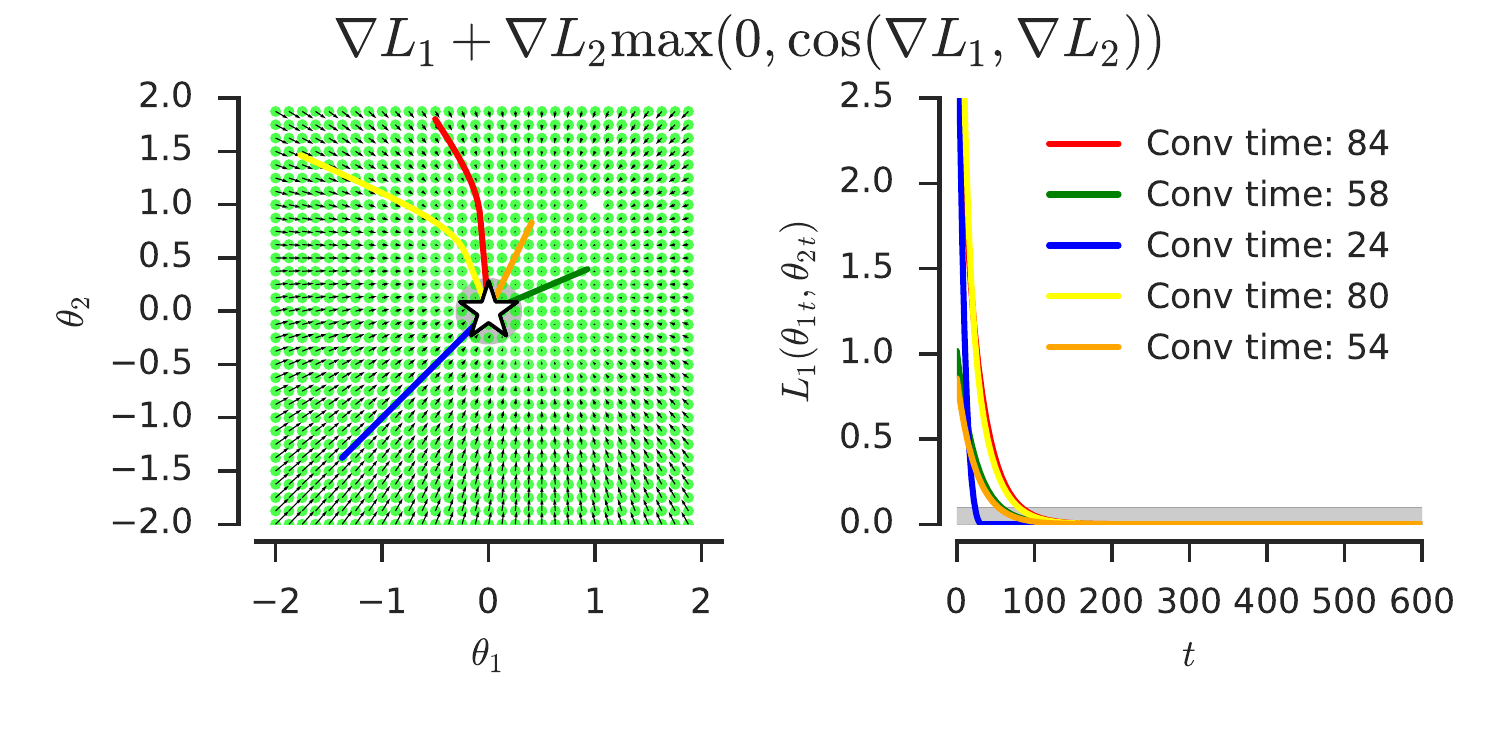}
    \includegraphics[width=\sfigwidth]{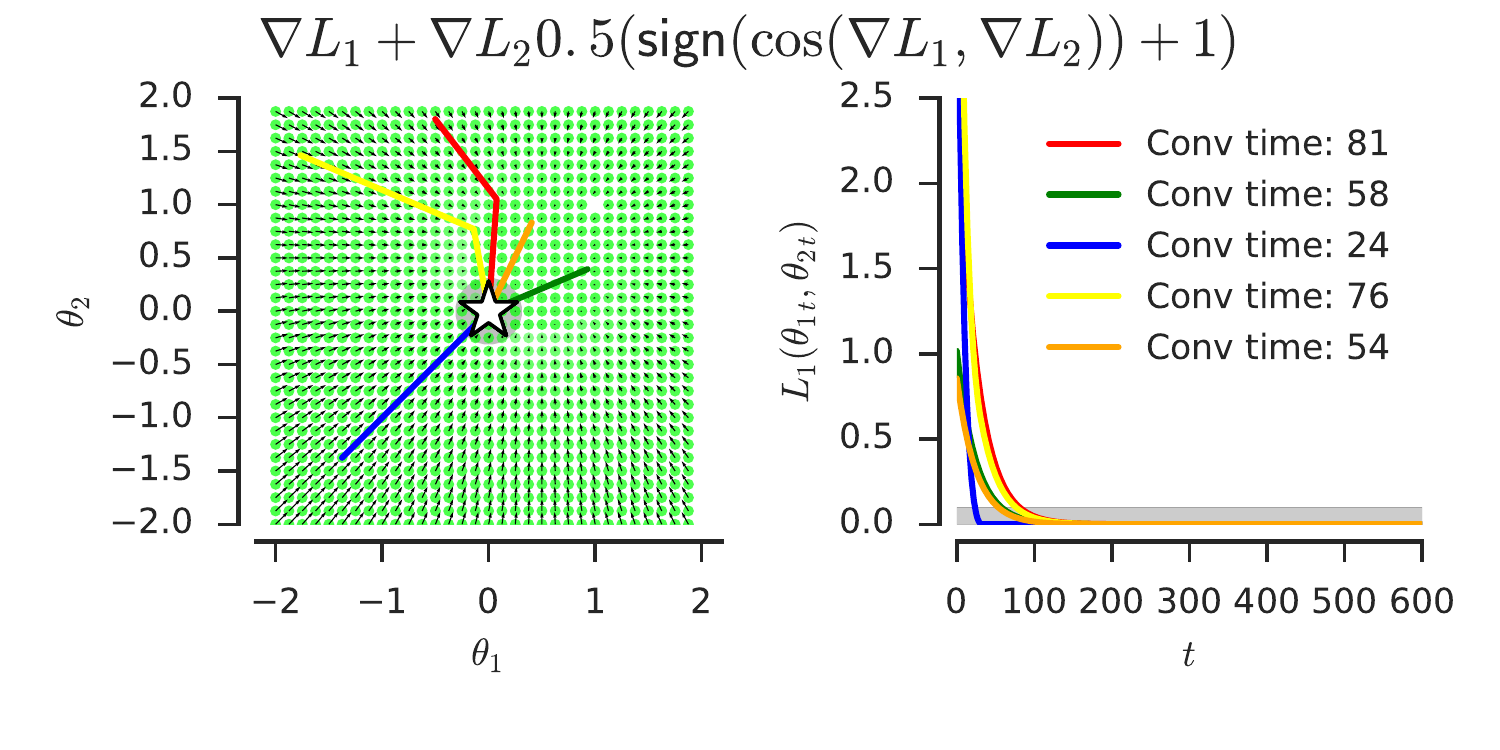}
    \includegraphics[width=\sfigwidth]{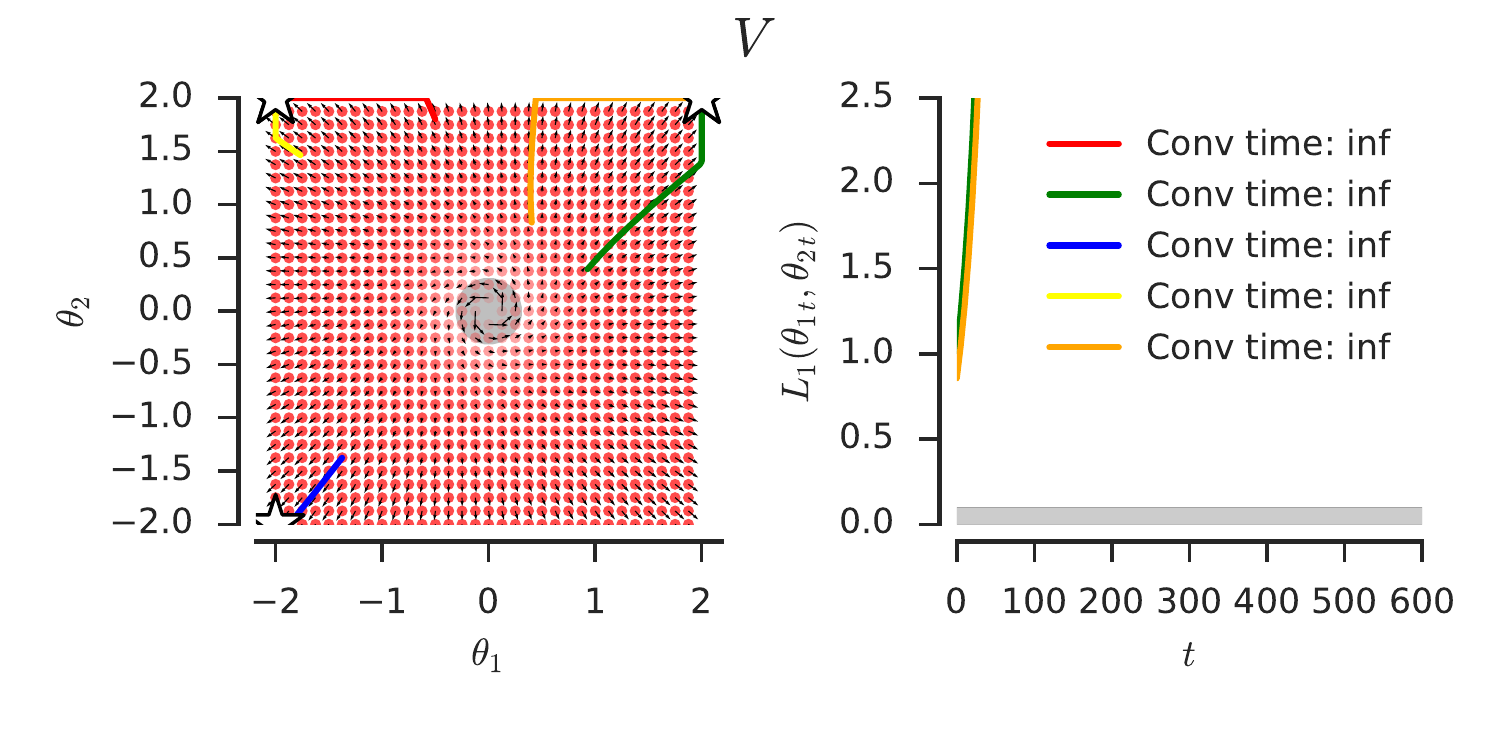}
    \includegraphics[width=\sfigwidth]{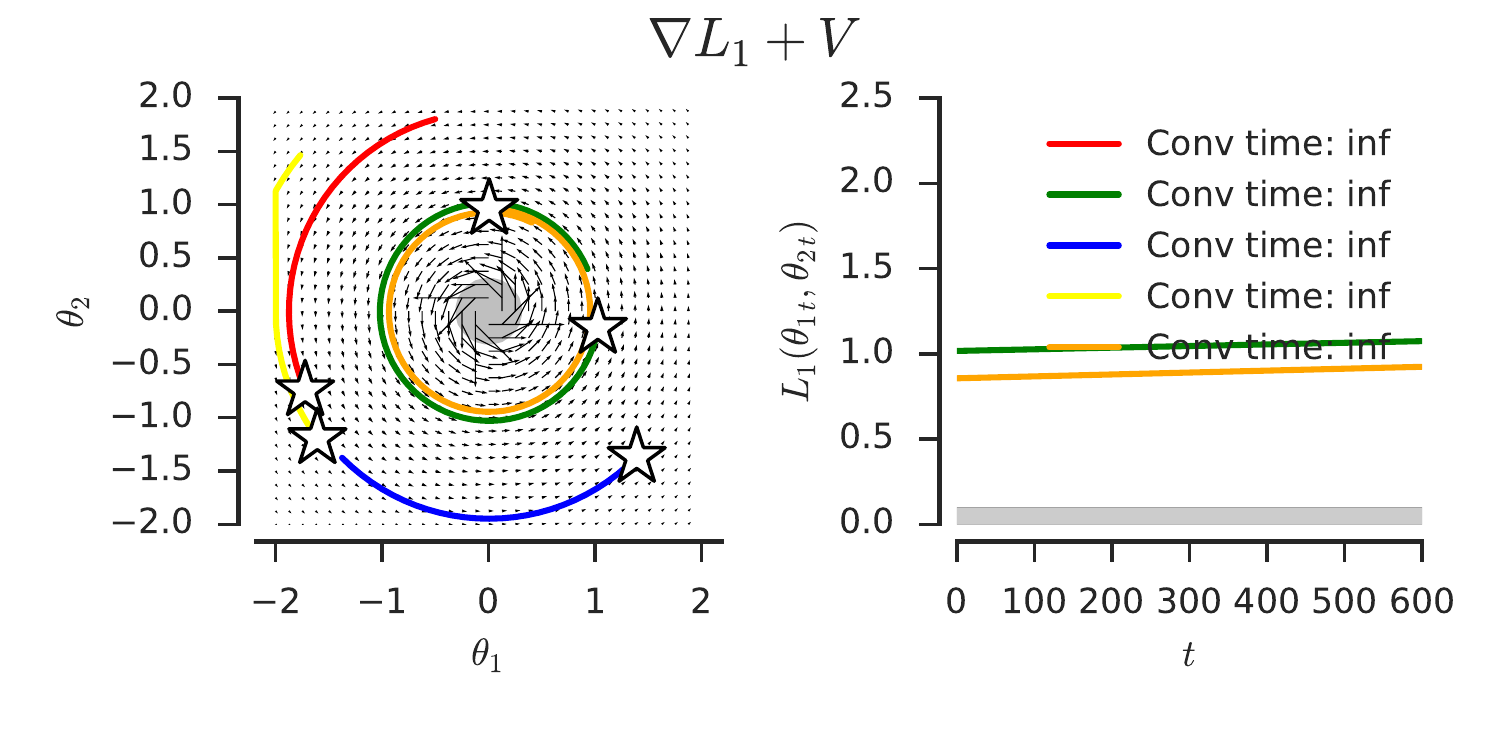}
    \includegraphics[width=\sfigwidth]{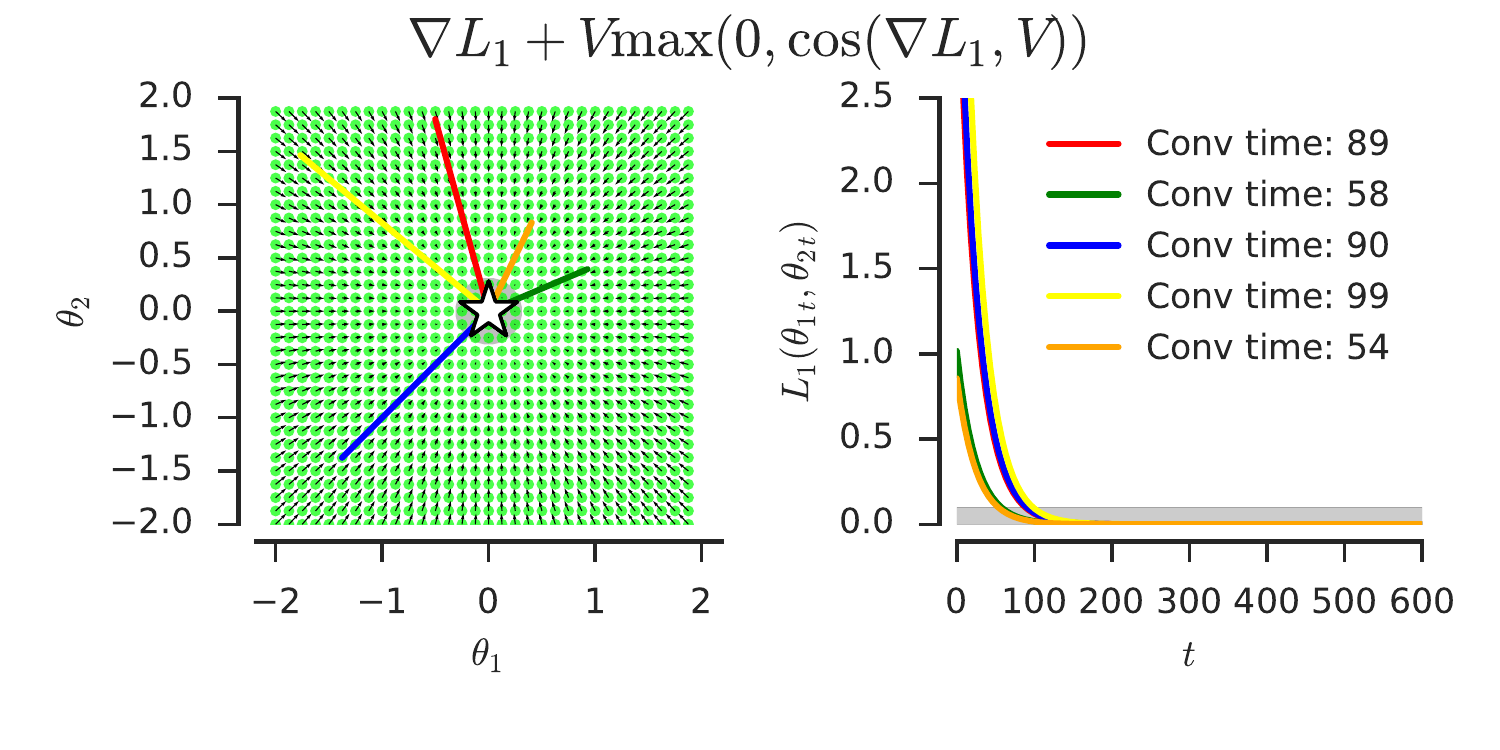}
    \includegraphics[width=\sfigwidth]{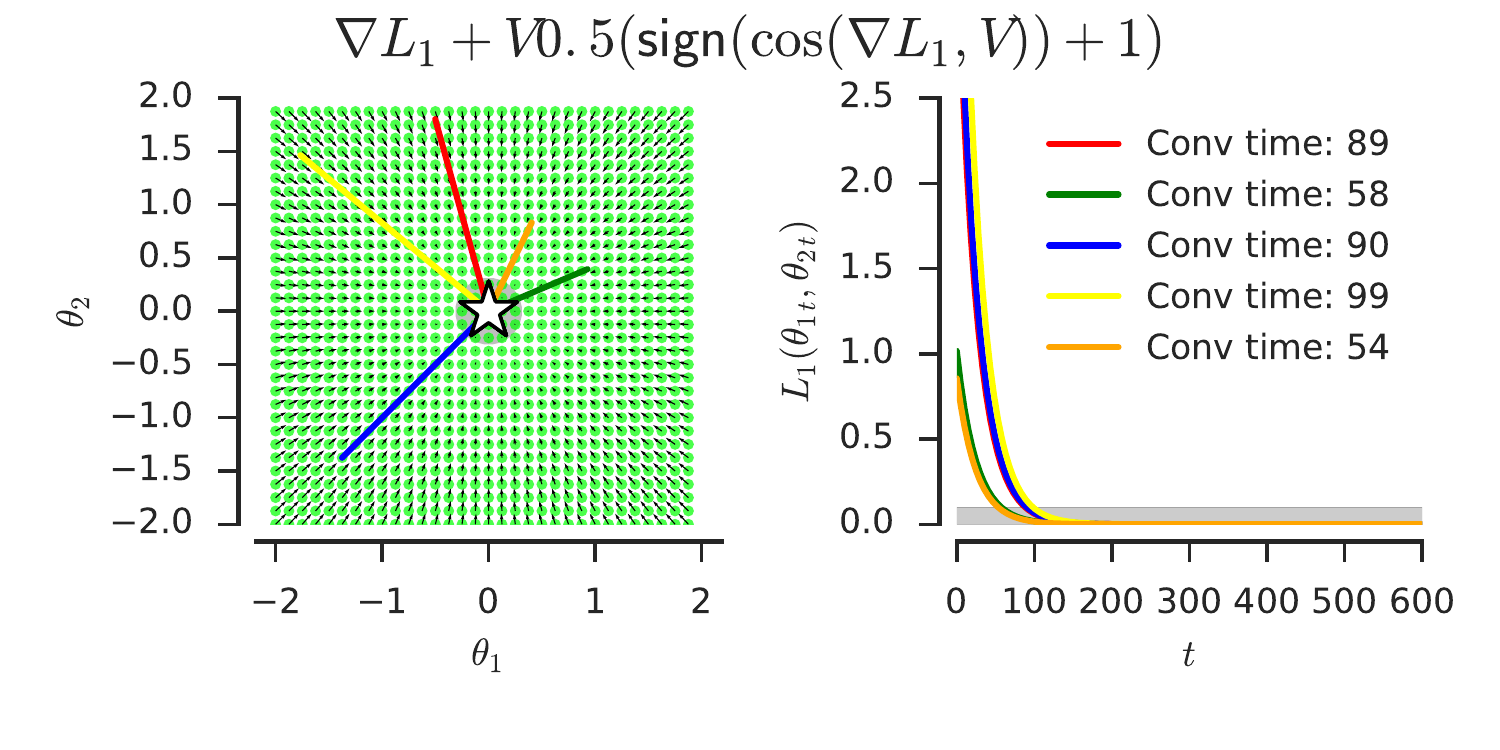}
    \caption{
    Positive example optimization for $L_1(\theta_1,\theta_2)=\theta_1^2+\theta_2^2$, $L_2(\theta_1,\theta_2) = (\theta_1-1)^2 + (\theta_2-1)^2$ and $V (\theta_1, \theta_2) = [ -\tfrac{\theta_2}{\theta_1^2+\theta_2^2}-2\theta_1,\tfrac{\theta_1}{\theta_1^2+\theta_2^2} -2\theta_2 ]$ where the proposed method speeds up the process (compared on all runs). Each colored trajectory represents one optimization run with random initial position. Star represents the convergence point. All experiments use steepest descent method and run 600 iterations with a constant step size of 0.01. Convergence time is defined as number of steps needed to get below 0.1 loss of $L_1$ (gray region). Color of each point represents its alignment with $\nabla L_1$ (green---positive alignment, red---negative alignment, white---directions are perpendicular). In this example $L_2$ is helpful for $L_1$ as it reinforces good descent directions in most of the space. However, simple mixing is actually slowing optimization down (or makes it fail completely, see the second row), while the proposed methods (weighted and unweighted variants) converge faster (see the third row). When using non-conservative vector field $V$ one obtains lack of convergence (cyclic behaviour, see the fourth row), while the proposed merging still works well (see the last row). 
    }
    \label{fig:aux-positive}
\end{figure*}

\section{Positive and Negative Examples of Auxiliary Losses}\label{sec:failuremodes}

We show in figure \ref{fig:aux-positive} positive examples of where an auxiliary loss can help a main loss to converge faster. In addition, as mentioned in Section \ref{sec:cos-between-task} that an auxiliary task does not guarantee faster convergence, we discuss here a few potential issues of using cosine similarity of gradients to measure task similarity and show in Figure~\ref{fig:aux-negative} a negative example on where the auxiliary loss could slow-down the convergence of the main task. First, the method depends on being able to compute cosine between gradients. However, in deep learning we rarely are able to compute exact gradients in practice, we instead depend on their high variance estimators (mini-batches in supervised learning, or Monte Carlo estimators in RL). Consequently, estimating the cosine similarity might require additional tricks such as keeping moving averages of estimates. Second, adding additional task gradient in selected subset of iterates can lead to very bumpy surface from the perspective of optimizer, causing methods which keep track of gradient statistics/estimate higher order derivatives, can be less efficient. Finally, one can construct specific functions, where despite still minimizing the loss, one significantly slows down optimization process.
\begin{figure}[ht]
    \centering
    \includegraphics[width=0.48\textwidth]{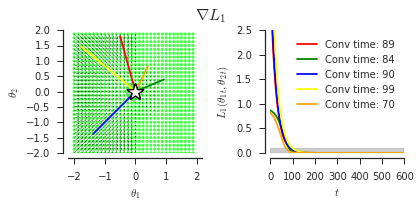}
    \includegraphics[width=0.48\textwidth]{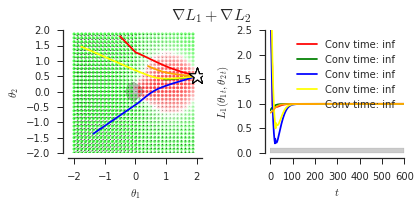}
    \includegraphics[width=0.48\textwidth]{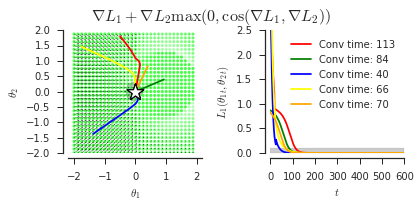}
    \includegraphics[width=0.48\textwidth]{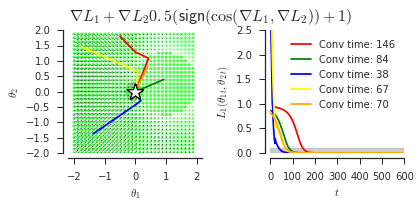}
    \caption{Negative example optimization for $L_1(\theta) = (\theta_1 < 0) ( \theta_1^2 + \theta_2^2 ) + (\theta_1>0)  \Bigl(1-\exp\bigl( -2(\theta_1^2 + \theta_2^2) \bigr)\Bigr)$ and $L_2(\theta) = (\theta_1 - 2)^2 + (\theta_2 - 0.5)^2$ where the proposed method slows down the process (compared on red runs). For the ease of presentation, we choose $L_1$, which is non-differentiable/smooth when $\theta=0$. But one can create any smooth functions with analogous properties. The core idea is, when there exists a flat region on the loss surface, the auxiliary lost tends to push the iterates to this region. Even though this move still decreases the loss (i.e.,  convergence is guaranteed), the optimization process will be slowed down.}
    \label{fig:aux-negative}
\end{figure}

\section{Weighted Version of Our Method}\label{sec:pseudocode}

Algorithm~\ref{alg:usecosine2} describes the \emph{weighted} version of our method.

\setcounter{algorithm}{1}
\begin{algorithm*}[ht]
 \caption{Weighted version of our method. %
 }
 \label{alg:usecosine2}
 \begin{algorithmic}[1]
 \State Initialize shared parameters $\paramshared$ and task specific parameters $\paramsone, \paramstwo$.
 randomly.
 \For{$\mathsf{iter}=1:\mathsf{max\_iter}$}
     \State Compute $\nabla_{\paramshared} \mainloss, \nabla_{\paramsone} \mainloss$, $\nabla_{\paramshared} \auxloss, \nabla_{\paramstwo} \auxloss$.
     \State Update $\paramsone$ and $\paramstwo$ using corresponding gradients
     \State Update $\paramshared$ using $\nabla_{\paramshared} \mainloss + \max(0, \cos(\nabla_{\paramshared} \mainloss, \nabla_{\paramshared} \auxloss)) \nabla_{\paramshared} \auxloss$
 \EndFor
 \end{algorithmic}
\end{algorithm*}

\section{Experimental Details} 
\label{sec:moreexps}

We present in this section experimental details for the ImageNet classification task, the RL gridworld task, and the RL Atari game task.  

\subsection{Identifying Near and Far Classes in ImageNet}
\label{sec:imagenet:classes:picking}

 \begin{figure}[th]
 \begin{center}
  \includegraphics[width=1\textwidth]{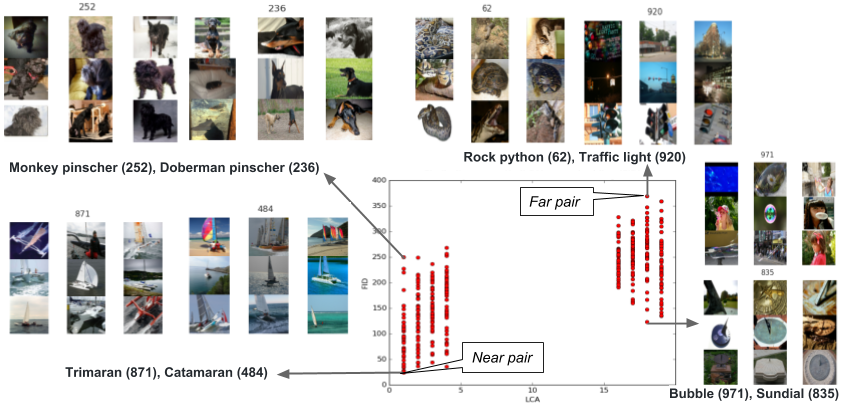}
 \caption{LCA ($x$-axis) versus FID ($y$-axis) as a ground truth for class similarity. The measurements reflect human intuition of class similarity; \emph{trimaran and catamaran} (bottom-left) are similar both visually and conceptually, whereas \emph{rock python and traffic light} (top-right) are dissimilar both visually and conceptually.
 }
 \label{fig:lca-fid}
 \end{center}
 \end{figure}

As a ground truth for class similarity, we identify pairs of ImageNet classes to be near or far using, \emph{lowest common ancestor (LCA)} and \emph{Frechet Inception Distance (FID)} \citep{fid}. 

ImageNet follows a tree hierarchy where each class is a leaf node. We define the distance between a pair of classes as at which tree level their LCA is found. In particular, there are 19 levels in the class tree, each leaf node (i.e. class) is considered to be level 0 while the root node is considered to be level 19. We perform bottom-up search for one pair of random sampled classes and find their LCA node---the class distance is then defined as the level number of this node.  For example, class $871$ (``trimaran'') and class $484$ (``catamaran'') has class distance 1 because their LCA is one level up. %

FID is used as a second measure of similarity. We obtain the image embedding of a pair of classes using the penultimate layer of a pre-trained ResNetV2-50 model \citep{he2016identity} and then compute the embedding distance using FID, defined in \cite{fid} as: 
\begin{align}
      \mathrm{FID} =  d^2\bigl((m_1, C_1), (m_2, C_2)\bigr) = \|m_1-m_2\|_2^2 + \mathrm{Tr}\bigl( C_1 + C_2 - 2(C_1C_2)^{1/2} \bigr).
  \end{align}
where $m_k, C_k$ denote the mean and covariance of the  embeddings from class $k$. 

We randomly sampled 50 pairs of classes for each level of $LCA=\{1, 2, 3, 4, 16, 17, 18, 19\}$ (400 pair of classes in total) and compute their FID. Figure~\ref{fig:lca-fid} shows a plot of LCA (x-axis) verses FID (y-axis) over our sampled class pairs. It can be seen that LCA and FID are (loosely) correlated and that they reflect human intuition of task similarity for some pairs. For example, \emph{trimaran and catamaran} (bottom-left) are similar both visually and conceptually, whereas \emph{rock python and traffic light} (top-right) are dissimilar both visually and conceptually. However, there are contrary examples where LCA disagrees with FID; \emph{monkey pinscher and doberman pinscher} (top-left) are visually dissimilar but conceptually similar, whereas \emph{bubble and sundial} (bottom-right) are visually similar but conceptually dissimilar. Per the observations, in subsequent experiments we pick class pairs that are \emph{\{Low LCA, Low FID\}} as \emph{near} pairs (e.g., trimaran and catamaran), and class pairs that are \emph{\{high LCA, high FID\}} as \emph{far} pairs (e.g., rock python and traffic light).

\subsection{Gridworld Experiments}
\label{sec:maze}
 We define a distribution over $15\times15$ gridworlds, where an agent observes its surrounding (up to 4 pixels away) and can move in 4 directions (with 10\% transition noise).
We randomly place walls (blocking movement) as well as 
two types of positive rewards: $+5$ and $+10$ points, both terminating an episode. There are also some negative rewards (both terminating and non-terminating) to make problem harder. 
In order to guarantee (expected) finite length of episodes we add fixed probability of 0.01 of transitioning to a non-rewarding terminal state.

For the sake of simplicity we use episode-level policy gradient \citep{reinforce} with value function baseline, with policies parametrized as logits $\paramshared$ of $\pi(a|s) = \tfrac{\exp( \theta_{s,a} )}{\sum_b \exp( \theta_{s,b})} \in [0,1]$, baselines as $B_s \in \mathbb{R}$, with fixed learning rate of $\alpha=0.01$, discount factor $\gamma=0.95$ and 10,000 training steps (states visited).

For this setup, the update rule for each sequence $\tau = \bigl((s_1, a_1, r_1), \dots (s_N, a_N, r_N)\bigr)$ is thus given by
$$
\Delta \paramshared =  \alpha \nabla_\paramshared \log \pi(a_\iterrltime | s_\iterrltime ) \left [ \sum_{\iterrl=0}^{N-\iterrltime} r_{\iterrltime+\iterrl}  - B_{s_\iterrltime} \right ] = \alpha G^\iter \;\;\;\;\;\;\;\;
\Delta B_{s_\iterrltime} =  - \alpha  \nabla_{B_{s_\iterrltime}} (B_{s_\iterrltime} -  \sum_{\iterrl=0}^{N-\iterrltime} r_{\iterrltime+\iterrl})^2.
$$

In order to make use of expert policies for $\auxtask$ we define auxiliary loss as a distillation loss, which is just a per-state cross-entropy between teacher's and student's distributions over actions. 
If we just add gradients estimated by policy gradient, and the ones given by distillation,
the update is given by 
$$
\Delta \paramshared =  \alpha \left [ G^\iter - \nabla_\paramshared \text{H}^\times( \pi^\text{Q}(\cdot | s_\iterrltime ) \| \pi(\cdot | s_\iterrltime) ) \right ]
=  \alpha  [ G^\iter +   \sum_a \pi^\text{Q}(a | s_\iterrltime ) \nabla_\paramshared \log \pi(a | s_\iterrltime)  ],
$$
where $V^\iter=\sum_a \pi^\text{Q}(a | s_\iterrltime ) \nabla_\paramshared \log \pi(a | s_\iterrltime) $ 
and $\text{H}^\times(p,q) = -\sum_k p_k \log q_k$ is the cross entropy.

However, if we use the proposed gradient cosine similarity, we get the following update
$$
\Delta \paramshared
= \alpha \left [ G^\iter + V^\iter \bigl( 2 \cdot  \sign( \cos(G^\iter, V^\iter) ) - 1\bigr) \right ].
$$ 
This get a significant boost to performance, and policies that score on average \textbf{3} points after 10,000 steps and obtain baseline performance after just one third of steps. Figure~\ref{fig:distillcos} shows a depiction of the task and an example solution. 

\begin{figure}[t]
    \centering
    \includegraphics[width=0.24\textwidth]{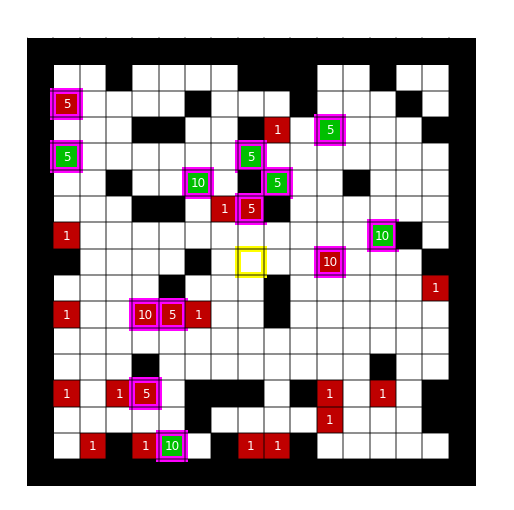}
    \includegraphics[width=0.24\textwidth]{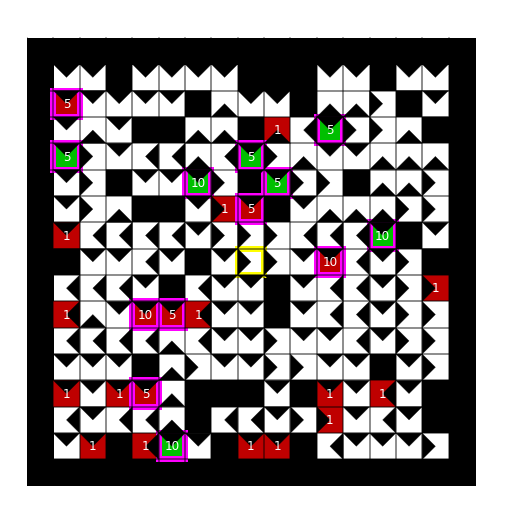}
    \includegraphics[width=0.24\textwidth]{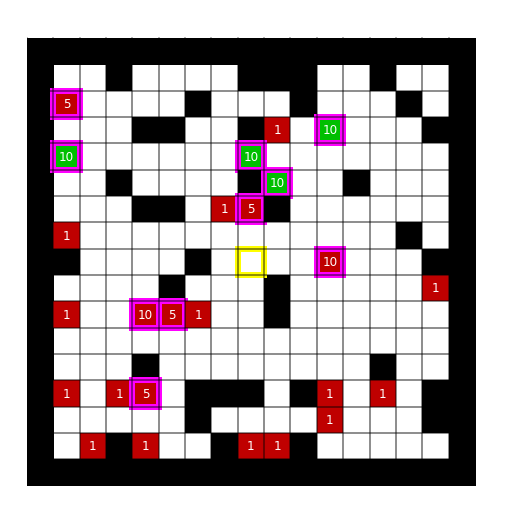}
    \includegraphics[width=0.24\textwidth]{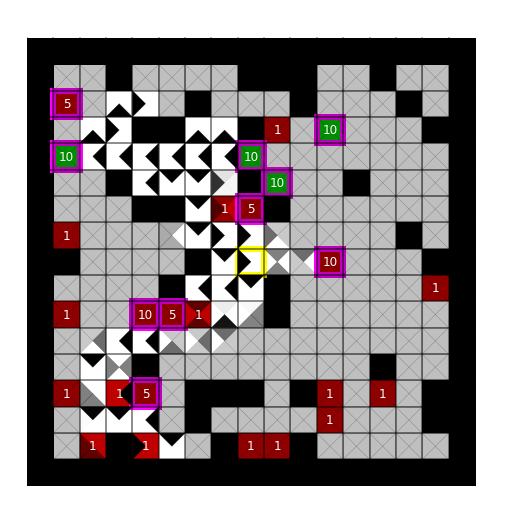}
    \caption{\textbf{Left most}: Initial task $\maintask$, yellow border denotes starting point and violet ones terminating states. Red states are penalizing with the value in the box while the green ones provide positive reward. \textbf{Middle Left}: Solution found by a single run of Q-learning with uniform exploration policy. \textbf{Middle Right}: Transformed task $\auxtask$. \textbf{Right most}: Solution found by gradient cosine similarity driven distillation with policy gradient.
    }
    \label{fig:distillcos}
\end{figure}

\subsection{Atari Experiments}\label{sec:atari}

For all Atari game experiments, we use a convolutional architecture as in previous work \citep{impala, popart, mnih2015human, mnih2016asynchronous}, trained with batched actor-critic with the V-trace algorithm \citep{impala}. We use a learning rate of $0.0006$ and an entropy cost of $0.01$ for all experiments, with a batch size of $32$ and $200$ parallel actors. 

For the single game experiment, Breakout, we use 0.02 for the threshold on the cosine similarity and, for technical reasons we ended up computing the cosine distance on a per-layer basis and then averaged. %
We additionally need to do a moving average of the cosine over time ($0.999c^\iterminusone + 0.001c^\iter$) to ensure there are no sudden spikes in the weighting due to noisy gradients.
Same setting is used for the multi-task experiment, just that the threshold is set to 0.01.

\section{Gradient Cosine Similarity in High Dimensions}\label{sec:cosine-high-dimensions}

\begin{figure}[ht]
\begin{center}
\includegraphics[width=0.9\textwidth]{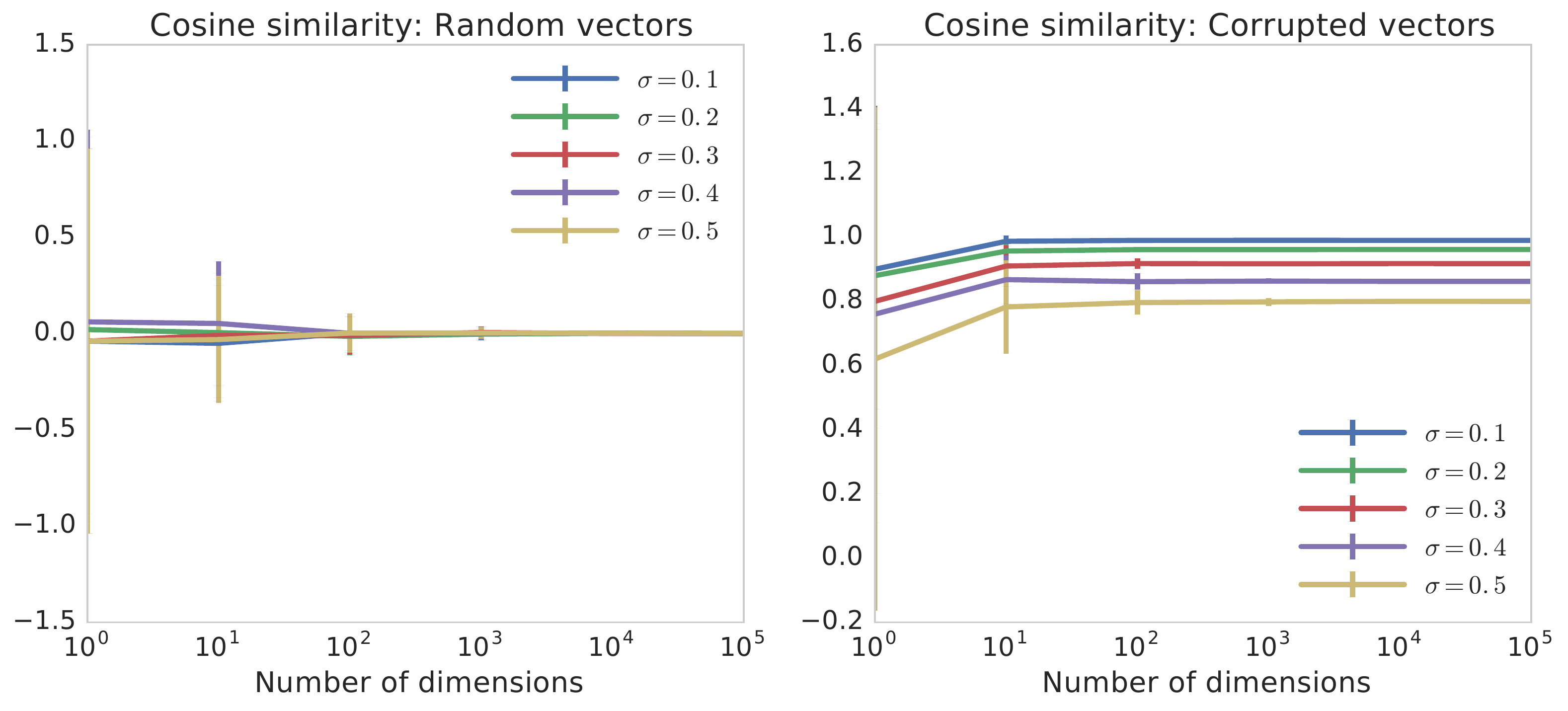}
\caption{Cosine similarity as a function of dimensionality. \emph{Left}: we generate two random vectors $\theta_1$ and $\theta_2$ from a Gaussian distribution with zero mean and variance $\sigma^2$  and as expected, the cosine similarity drops to zero very quickly as the number of dimensions increases. \emph{Right}: we mimic a scenario where the true gradients of the main and auxiliary are aligned, however we observe only \emph{corrupted} noisy gradients which are noisy copies of the true underlying vector; we generate $\mu\sim \mathcal{N}(0,I_d)$ and generate $\theta_1\sim\mathcal{N}(\mu,\sigma I_d)$ and $\theta_2\sim\mathcal{N}(\mu,\sigma I_d)$. In this case, cosine similarity is larger in higher dimensions (as the inner product of the corruption noise goes to zero).
}
\label{fig:cosine-similarity-random-vs-corrupted}
\end{center}
\end{figure}

As discussed in Section \ref{sec:discussion} of the paper, we empirically show in Figure \ref{fig:cosine-similarity-random-vs-corrupted} that if two gradients are meant to be co-linear (e.g., tasks share some underlying structure), the cosine similarity can still be a potential way of measuring task similarity even in high dimensions.

\section{Overparametrization and reliability of gradient similarity globally}\label{sec:overparam}
As discussed in Section~\ref{sec:discussion} of the paper, in theory it is not clear that moving in an ascent direction of the main loss locally (e.g. for a small finite number of steps) has negative consequences on speed of convergence when the main loss is non-convex as in the case of neural networks. 

For example the auxiliary loss could drive the model outside of the convex bowl (basin of attraction) of a suboptimal minima, helping the neural network convert to a better solution for the main loss, even though locally it will move in an ascent direction on the main loss. In general if we think of the auxiliary loss as some form of regularization, one interpretation of regularization terms is that it is needed to prevent overfitting the data. That implies that it will, at some point in training, stop the model to move in the descent direction of the training error which would lead to memorization of the training set. 

However we observe that in practice, as a heuristic,
gradient similarity seems to work reasonably well.
While we do not have an understanding of the efficiency of our heuristic, we hypothesize that 
overparametrization of neural networks might play 
an important role in the explanation. For example, 
it is well understood that overparametrization makes it less likely for the loss surface to exhibit suboptimal local minima. Therefore the explicit case of needing to move in ascent direction in order to escape the basing of attraction of a suboptimal solution becomes less likely. The lack of bad local minima however does not exclude that moving in an ascent direction can potentially have positive effects globally. For example, one can imagine moving in a descent direction forcing learning to traverse a plateaux until it reaches a minima, while moving in a ascent direction leads to a steep valley. The minima reached on both direction could be connected (as most minima for neural networks seem to be \cite{garipov2018loss,draxler2018essentially}) and hence this is not a case of being stuck in a suboptimal solution. While such structures of the loss surface have not been investigated, and they could happen due to the non-linearity, we believe that overparametrization makes them less likely. We leave exploring this hypothesis as an open question for future work.

\end{document}